\documentclass[nohyperref]{article}

\usepackage{microtype}
\usepackage{graphicx}
\usepackage{booktabs} 

\usepackage{xcolor}
\definecolor{mydarkblue}{rgb}{0,0.08,0.45}
\usepackage[colorlinks, citecolor=mydarkblue]{hyperref}

\usepackage[accepted]{icml2023}

\usepackage{amsmath}
\usepackage{amssymb}
\usepackage{mathtools}
\usepackage{amsthm}
\usepackage{thmtools} 
\usepackage{thm-restate}
\usepackage{bbold} 

\usepackage[capitalize,noabbrev]{cleveref}

\theoremstyle{plain}
\newtheorem{theorem}{Theorem}[section]
\newtheorem{proposition}[theorem]{Proposition}
\newtheorem{lemma}[theorem]{Lemma}

\theoremstyle{definition}
\newtheorem{definition}[theorem]{Definition}

\theoremstyle{remark}

\usepackage{graphicx,float,subfig}

\DeclareMathOperator*{\argmin}{arg\,min}

\usepackage[textsize=tiny]{todonotes}

\icmltitlerunning{Sliced-Wasserstein on Symmetric Positive Definite Matrices for M/EEG Signals}

\begin{document}

\twocolumn[
\icmltitle{Sliced-Wasserstein on Symmetric Positive Definite Matrices for M/EEG Signals}



\icmlsetsymbol{equal}{*}

\begin{icmlauthorlist}

\icmlauthor{Clément Bonet}{equal,lmba}
\icmlauthor{Benoît Malézieux}{equal,inria}
\icmlauthor{Alain Rakotomamonjy}{criteo,rouen}
\icmlauthor{Lucas Drumetz}{imt}
\icmlauthor{Thomas Moreau}{inria}
\icmlauthor{Matthieu Kowalski}{lisn}
\icmlauthor{Nicolas Courty}{irisa}

\end{icmlauthorlist}


\icmlaffiliation{lisn}{Université Paris-Saclay, CNRS, LISN}
\icmlaffiliation{inria}{Université Paris-Saclay, Inria, CEA}
\icmlaffiliation{lmba}{Université Bretagne Sud, LMBA}
\icmlaffiliation{irisa}{Université Bretagne Sud, IRISA}
\icmlaffiliation{imt}{IMT Atlantique, Lab-STICC}
\icmlaffiliation{criteo}{Criteo AI Lab}
\icmlaffiliation{rouen}{Université de Rouen, LITIS}

\icmlcorrespondingauthor{Clément Bonet}{clement.bonet@univ-ubs.fr}
\icmlcorrespondingauthor{Benoît Malézieux}{benoit.malezieux@inria.fr}

\icmlkeywords{Machine Learning, ICML}

\vskip 0.3in
]



\printAffiliationsAndNotice{\icmlEqualContribution} 

\begin{abstract}
When dealing with electro or magnetoencephalography records, many supervised prediction tasks are solved by working with covariance matrices to summarize the signals.
Learning with these matrices requires using Riemanian geometry to account for their structure.
In this paper, we propose a new method to deal with distributions of covariance matrices and demonstrate its computational efficiency on M/EEG multivariate time series.
More specifically, we define a Sliced-Wasserstein distance between measures of symmetric positive definite matrices that comes with strong theoretical guarantees.
Then, we take advantage of its properties and kernel methods to apply this distance to brain-age prediction from MEG data and compare it to state-of-the-art algorithms based on Riemannian geometry.
Finally, we show that it is an efficient surrogate to the Wasserstein distance in domain adaptation for Brain Computer Interface applications.
\end{abstract}

\section{Introduction}

Magnetoencephalography and electroencephalography (M/EEG) are non-invasive techniques for recording the electrical activity of the brain \citep{hamalainen1993magnetoencephalography}.
The data consist of multivariate time series output by sensors placed around the head, which capture the intensity of the magnetic or electric field with high temporal resolution.
Those measurements provide 
information on cognitive processes as well as the biological state of a subject.

Successful machine learning (ML) techniques that deal with M/EEG data often rely on covariance matrices estimated from band-passed filtered signals in several frequency bands \citep{blankertz2007optimizing}.
The main difficulty that arises when processing such covariance matrices is that the set of symmetric positive definite (SPD) matrices is not a linear space, but a Riemannian manifold \citep{bhatia2009positive, bridson2013metric}.
Therefore, specific algorithms have to be designed to take into account the non Euclidean structure of the data.
The usage of Riemannian geometry on SPD matrices has become increasingly popular in the ML community \citep{huang2017riemannian, chevallier2017kernel, ilea2018covariance, brooks2019riemannian}. 
In particular, these tools have proven to be very effective on prediction tasks with M/EEG data in Brain Computer Interface (BCI) applications \citep{barachant2011multiclass, barachant2013classification, gaur2018multi} or more recently in brain-age prediction \citep{sabbagh2019manifold, sabbagh2020predictive, engemann2022reusable}. 
As covariance matrices sets from M/EEG data are often modeled as samples from a probability distribution -- for instance in domain adaptation for BCI \citep{yair2019domain} -- it is of great interest to develop efficient tools that work directly on those distributions.

\looseness-1
Optimal transport (OT) \citep{villani2009optimal,peyre2019computational} provides a powerful theoretical framework and computational toolbox to compare probability distributions while respecting the geometry of the underlying space.
It is well defined on Riemannian manifolds \citep{mccann2001polar, cui2019spherical,alvarez2020unsupervised} and in particular on the space of SPD matrices that is considered in M/EEG learning tasks \citep{brigant2018optimal, yair2019domain, ju2022deep}.
The original OT problem defines the Wasserstein distance which has a super cubic complexity \emph{w.r.t} samples.
To alleviate the computational burden, different alternatives were proposed such as adding an entropic regularization \citep{cuturi2013sinkhorn} or computing the distance between mini-batches \citep{fatras2019learning}.
Another popular alternative is the Sliced-Wasserstein distance (SW) \citep{rabin2011wasserstein} which computes the average of the Wasserstein distance between one-dimensional projections.
SW has recently received a lot of attention as it significantly reduces the computational burden while preserving topological properties of Wasserstein \citep{bonnotte2013unidimensional,nadjahi2020statistical, bayraktar2021strong}.
Moreover, \citet{kolouri2016sliced, meunier2022distribution} have shown that, as opposed to Wasserstein, SW allows to properly extend kernel methods to data-sets of distributions with very efficient computation of the kernel matrix. 
This opens the way to new regression and classification methods.
However, the initial construction of SW is restricted to Euclidean spaces.
Thus, a new line of work focuses on its extension to specific manifolds \citep{rustamov2020intrinsic, bonet2022hyperbolic, bonet2022spherical}.

\paragraph{Contributions.}
In order to benefit from the advantages of SW in the context of M/EEG, we propose an SW distance on the manifold of SPD matrices and evaluate its efficiency on two prediction tasks.
\begin{itemize}\itemsep-.2em
    \item We introduce an SW discrepancy between measures of symmetric positive definite matrices ($\mathrm{SPDSW}$),
    and provide a well-founded numerical approximation.
    \item We derive theoretical results, including topological, statistical, and computational properties.
    In particular, we prove that $\mathrm{SPDSW}$ is a distance topologically equivalent to the Wasserstein distance in this context.
    \item \looseness=-1 We extend the distribution regression with SW kernels to the case of SPD matrices, apply it to brain-age regression with MEG data, and show that it performs better than other methods based on Riemannian geometry.
    \item We show that $\mathrm{SPDSW}$ is an efficient surrogate to the Wasserstein distance in domain adaptation for BCI.
\end{itemize}

\section{Sliced-Wasserstein on SPD matrices}

In this section, we introduce an SW discrepancy on SPD matrices and provide a theoretical analysis of its properties and behavior.
The proofs are deferred to \cref{sec:proofs}.

\subsection{Euclidean Sliced-Wasserstein distance}

For $\mu,\nu\in\mathcal{P}_p(\mathbb{R}^d)$ two measures with finite moments of order $p\ge 1$, the Wasserstein distance is defined as
\begin{equation} \label{eq:wasserstein}
    W_p^p(\mu,\nu) = \inf_{\gamma\in\Pi(\mu,\nu)}\ \int \|x-y\|_2^p\ \mathrm{d}\gamma(x,y)\enspace ,
\end{equation}
where $\Pi(\mu,\nu) = \{\gamma\in\mathcal{P}(\mathbb{R}^d\times \mathbb{R}^d), \ \pi^1_\#\gamma=\mu,\ \pi^2_\#\gamma=\nu\}$ denotes the set of couplings between $\mu$ and $\nu$, $\pi^1:(x,y)\mapsto x$ and $\pi^2:(x,y)\mapsto y$ the projections on the first and second coordinate and $\#$ is the push-forward operator, defined as a mapping on all borelian $A\subset \mathbb{R}^d$, such that $T_\#\mu(A) = \mu(T^{-1}(A))$. For practical ML applications, this distance is computed between two empirical distributions with $n$ samples and the main bottleneck consists in solving the linear program (\ref{eq:wasserstein}).
Its computational complexity is $O(n^3\log n)$ \citep{pele2009fast} which is expensive for large scale applications.

While computing (\ref{eq:wasserstein}) is costly in general, it can be computed efficiently for problems where $d=1$, as it admits the following closed-form \citep[Remark 2.30]{peyre2019computational}
\begin{equation}
    W_p^p(\mu,\nu) = \int_0^1 |F_\mu^{-1}(u)-F_\nu^{-1}(u)|^p\ \mathrm{d}u\enspace ,
\end{equation}
where $F_\mu^{-1}$ and $F_\nu^{-1}$ are the quantile functions of $\mu$ and $\nu$.
By computing order statistics, this can be approximated from samples in $O(n\log n)$.

This observation motivated the construction of the SW distance \citep{rabin2011wasserstein,bonneel2015sliced} which is defined as the average of the Wasserstein distance between one dimensional projections of the measures in all directions, \emph{i.e.} for $\mu,\nu\in\mathcal{P}_p(\mathbb{R}^d)$,
\begin{equation}
    \mathrm{SW}_p^p(\mu,\nu) = \int_{S^{d-1}} W_p^p(t^\theta_\#\mu,t^\theta_\#\nu)\ \mathrm{d}\lambda(\theta)\enspace ,
\end{equation}
where $\lambda$ is the uniform distribution on the sphere $S^{d-1}=\{\theta\in \mathbb{R}^d,\ \|\theta\|_2=1\}$ and $t^\theta$ is the coordinate of the projection on the line $\mathrm{span}(\theta)$, \emph{i.e.} $t^\theta(x) = \langle x,\theta\rangle$ for $x\in \mathbb{R}^d$, $\theta\in S^{d-1}$.
This distance has many advantages, motivating its use in place of the Wasserstein distance.
First, it can be approximated in $O\big(Ln(d+\log n)\big)$ with $L$ projections and a Monte-Carlo method.
Moreover, it is topologically equivalent to the Wasserstein distance as it also metrizes the weak convergence \citep{nadjahi2019asymptotic}, and its sample complexity is independent of the dimension \citep{nadjahi2020statistical} as opposed to Wasserstein.
Finally, it is a Hilbertian metric and it can be used to define kernels over probability distributions \citep{kolouri2016sliced, carriere2017sliced, meunier2022distribution}.
This is particularly interesting for regression or classification over data-sets of distributions, as we will see in \cref{sec:brain_age} for brain-age prediction.

\subsection{Background on SPD matrices} \label{sec:bg_spd}

Let $S_d(\mathbb{R})$ be the set of symmetric matrices of $\mathbb{R}^{d \times d}$, and $S_d^{++}(\mathbb{R})$ be the set of SPD matrices of $\mathbb{R}^{d \times d}$, \emph{i.e.} matrices $M\in S_d(\mathbb{R})$ satisfying
\begin{equation}
    \forall x\in\mathbb{R}^d\setminus\{0\},\ x^T M x > 0 \enspace .
\end{equation}
$S_d^{++}(\mathbb{R})$ is a Riemannian manifold \citep{bhatia2009positive}, meaning that it behaves locally as a linear space, called a tangent space.
Each point $M \in S_d^{++}(\mathbb{R})$ defines a tangent space $\mathcal{T}_M$, which can be given an inner product $\langle \cdot, \cdot \rangle_M : \mathcal{T}_M \times \mathcal{T}_M \rightarrow \mathbb{R}$, and thus a norm.
The choice of this inner-product induces different geometry on the manifold.
One example is the geometric and Affine-Invariant metric \citep{pennec2006riemannian}, where the inner product is defined as
\begin{equation}
    \begin{aligned}
        &\forall M \in S_d^{++}(\mathbb{R}),\ A,B\in \mathcal{T}_M, \\
        &\langle A,B\rangle_M = \mathrm{Tr}(M^{-1}AM^{-1}B) \enspace .
    \end{aligned}
\end{equation}
Denoting by $\mathrm{Tr}$ the Trace operator, the corresponding geodesic distance $d_{AI}(\cdot,\cdot)$ is given by
\begin{equation}
    \forall X, Y \in S_d^{++}(\mathbb{R}), d_{AI}(X,Y) = \sqrt{\mathrm{Tr}(\log(X^{-1}Y)^2)}\enspace .
\end{equation}
Another example is the Log-Euclidean metric \citep{arsigny2005fast,arsigny2006log} for which, 
\begin{equation}
    \begin{aligned}
        &\forall M \in S_d^{++}(\mathbb{R}),\ A,B\in \mathcal{T}_M, \\
        &\langle A,B\rangle_M = \langle D_M\log A, D_M \log B\rangle \enspace ,
    \end{aligned}
\end{equation}
with $\log$ the matrix logarithm and $D_M\log A$ the directional derivative of the $\log$ at $M$ along $A$ \citep{huang2015log}.
This definition provides another geodesic distance \citep{arsigny2006log}
\begin{equation}
    \forall X, Y \in S_d^{++}(\mathbb{R}), \ d_{LE}(X,Y) = \|\log X - \log Y\|_F \enspace ,
\end{equation}
which is simply an Euclidean distance in $S_d(\mathbb{R})$ in this case.
We will use the Log-Euclidean metric in the following, as it is simpler and faster to compute while being a good first order approximation of the Affine-Invariant metric \citep{arsigny2005fast, pennec2020manifold}.
In this case, the geodesic between $X, Y \in S_d^{++}(\mathbb{R})$ is $t \in \mathbb{R} \mapsto \exp((1-t) \log X + t\log Y)$.
$\log$ is a diffeomorphism from $S_d^{++}(\mathbb{R})$ to $S_d(\mathbb{R})$, whose inverse is $\exp$.
Thus, the geodesic line going through $A \in S_d(\mathbb{R})$ and the origin of $S_d(\mathbb{R})$ is $\mathcal{G}_A = \{\exp(tA),\ t\in\mathbb{R}\}$.
To span all such geodesics, we can restrict to $A$ with unit Frobenius norm, i.e. $\| A \|_F = 1$.

\subsection{Construction of $\mathrm{SPDSW}$}

On a Euclidean space, the SW distance is defined by averaging the Wasserstein distance between the distributions projected over all possible straight lines passing through the origin. As $S_d^{++}(\mathbb{R})$ with Log-Euclidean metric is a geodesically complete Riemannian manifold, \emph{i.e.} there exists a geodesic curve between each couple of points and each geodesic curve can be extended to $\mathbb{R}$, a natural generalization of SW on this space can be obtained by averaging the Wasserstein distance between distributions projected over all geodesics passing through the origin $I_d$.
\begin{figure}[t]
    \centering
    \includegraphics[width=1\columnwidth]{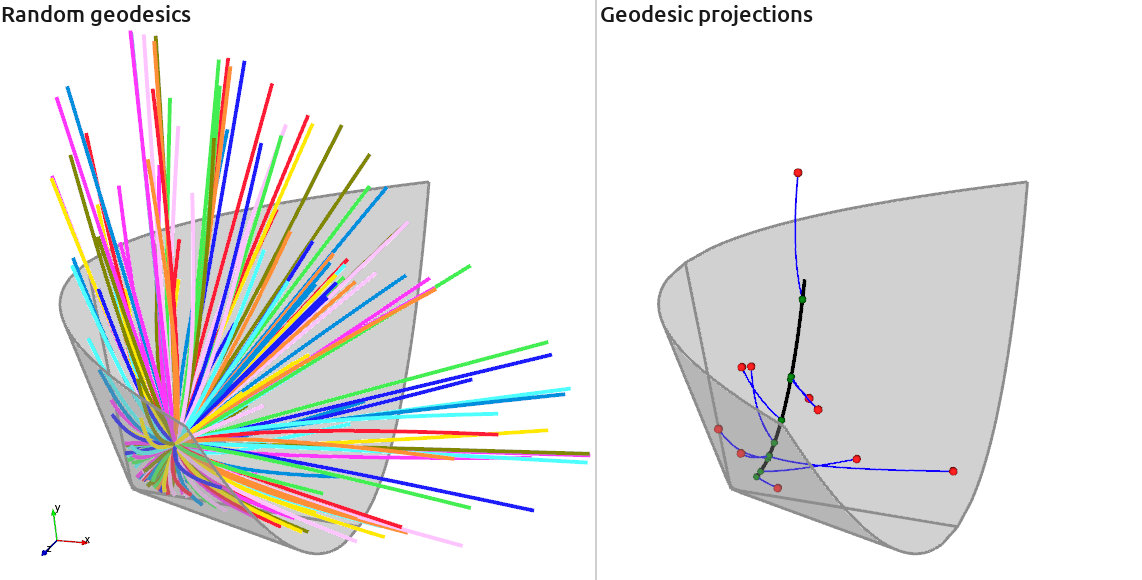}
    \caption{({\bf Left}) Random geodesics drawn in  $S_2^{++}(\mathbb{R})$. ({\bf Right}) Projections (green points) of covariance matrices (depicted as red points) over one geodesic (in black) passing through $I_2$ along the Log-Euclidean geodesics (blue lines).}
    \label{fig:proj}
    \vspace{-10pt}
\end{figure}

To construct $\mathrm{SPDSW}$, we need several ingredients. First, it is required to find the projection onto a geodesic $\mathcal{G}_A$ passing through $I_d$ where $A\in S_d(\mathbb{R})$. Such projection $P^{\mathcal{G}_A}$ can be obtained as follows
\begin{equation}
    \forall M\in S_d^{++}(\mathbb{R}),\ P^{\mathcal{G}_A}(M) = \argmin_{X\in\mathcal{G}_A}\ d_{LE}(X,M) \enspace ,
\end{equation}
and we provide the closed-form in \cref{prop:geodesic_projection}.

\begin{restatable}{proposition}{projGeodesic}
    \label{prop:geodesic_projection}
     Let $A\in S_d(\mathbb{R})$ with $\|A\|_F = 1$, and let $\mathcal{G}_A$ be the associated geodesic line.
     Then, for any $M\in S_d^{++}(\mathbb{R})$, the geodesic projection on $\mathcal{G}_A$ is
    \begin{equation}
        P^{\mathcal{G}_A}(M) = \exp\big(\mathrm{Tr}(A\log M) A\big)\enspace .
    \end{equation}
\end{restatable}
\vskip-.7em
Then, the coordinate of the projection on $\mathcal{G}_A$ can be obtained by giving an orientation to $\mathcal{G}_A$ and computing the distance between $P^{\mathcal{G}_A}(M)$ and the origin $I_d$, as follows
\begin{equation}
    \label{eqn:coordinate}
    t^A(M) = \mathrm{sign}(\langle \log M, A\rangle_F) d_{LE}(P^\mathcal{G}(M), I_d) \enspace .
\end{equation}
The closed-form expression is given by \cref{prop:geodesic_coordinate}.

\begin{restatable}{proposition}{coordinateGeodesic}
    \label{prop:geodesic_coordinate}
     Let $A\in S_d(\mathbb{R})$ with $\|A\|_F = 1$, and let $\mathcal{G}_A$ be the associated geodesic line.
     Then, for any $M\in S_d^{++}(\mathbb{R})$, the geodesic coordinate on $\mathcal{G}_A$ is
    \begin{equation}
        t^A(M) = \langle A, \log M\rangle_F = \mathrm{Tr}(A\log M)\enspace .
    \end{equation}
\end{restatable}
\vskip-0.7em
These two properties give a closed-form expression for the Riemannian equivalent of one-dimensional projection in a Euclidean space.
Note that coordinates on the geodesic might also be found using Busemann coordinates, similarly to the construction proposed by \citet{bonet2022hyperbolic}, and that they actually coincide here.
We add more details in \cref{sec:proofs}.
In \cref{fig:proj}, we illustrate the projections of matrices $M\in S_2^{++}(\mathbb{R})$ embedded as vectors $(m_{11},m_{22},m_{12})\in\mathbb{R}^3$.
$S_2^{++}(\mathbb{R})$ is an open cone and we plot the projections of random SPD matrices on geodesics passing through $I_2$.

We are now ready to define an SW discrepancy on measures in $\mathcal{P}_p(S_d^{++}(\mathbb{R}))=\{\mu\in\mathcal{P}(S_d^{++}(\mathbb{R})),\ \int d_{LE}(X, M_0)^p\ \mathrm{d}\mu(X)<\infty,\ M_0\in S_d^{++}(\mathbb{R})\}$.
\begin{definition}
    \label{def:swspd}
    Let $\lambda_S$ be the uniform distribution on $\{A\in S_d(\mathbb{R}),\ \|A\|_F=1\}$. Let $p\ge 1$ and $\mu,\nu\in\mathcal{P}_p(S_d^{++}(\mathbb{R}))$, then the $\mathrm{SPDSW}$ discrepancy is defined as 
    \begin{equation}
        \mathrm{SPDSW}_p^p(\mu,\nu) = \int_{S_d(\mathbb{R})} W_p^p(t^A_\#\mu, t^A_\#\nu)\ \mathrm{d}\lambda_S(A) \enspace .
    \end{equation}
\end{definition}
\vskip-.7em
As shown by the definition, being able to sample from
$\lambda_S$ is the cornerstone of the computation of SPDSW.  
In \cref{lemma:uniform_distribution}, we propose a practical way of 
uniformly sampling a symmetric matrix $A$.
More specifically, we sample an orthogonal matrix $P$ and a diagonal matrix $D$ of unit norm and compute $A=PDP^T$ which is a symmetric matrix of unit norm.
This is equivalent to sampling from $\lambda_S$ as the measures are equal up to a normalization factor $d!$ which represents the number of possible permutations of the columns of $P$ and $D$ for which $PDP^T=A$.
\begin{restatable}{lemma}{uniformDistribution}
    \label{lemma:uniform_distribution}
    Let $\lambda_O$ be the uniform distribution on $\mathcal{O}_d = \{P \in \mathbb{R}^{d \times d}, P^TP = PP^T = I\}$ (Haar distribution), and $\lambda$ be the uniform distribution on $S^{d-1} = \{\theta \in \mathbb{R}^d, \|\theta\|_2=1\}$.
    Then $\lambda_S \in \mathcal{P}(S_d(\mathbb{R}))$, defined such that $\forall \ A = P \mathrm{diag}(\theta) P^T \in S_d(\mathbb{R})$, $\mathrm{d}\lambda_S(A) = d! \ \mathrm{d}\lambda_O(P) \mathrm{d}\lambda(\theta)$, is the uniform distribution on $\{A\in S_d(\mathbb{R}),\ \|A\|_F=1\}$.
\end{restatable}
Then, the coordinate of the projection on the geodesic $\mathcal{G}_A$ is provided by $t^A(\cdot) = \mathrm{Tr}(A \log \cdot)$ defined in \cref{prop:geodesic_coordinate}.
The Wasserstein distance is easily computed using order statistics, and this leads to a natural extension of the SW distance in $S_d^{++}(\mathbb{R})$.
There exists a strong link between $\mathrm{SW}$ on distributions in $\mathbb{R}^{d \times d}$ and $\mathrm{SPDSW}$.
Indeed, \cref{prop:equivalence_swlog} shows that $\mathrm{SPDSW}$ is equal to a variant of $\mathrm{SW}$ where projection parameters are sampled from unit norm matrices in $S_d(\mathbb{R})$ instead of the unit sphere, and where the distributions are pushed forward by the $\log$ operator.
\begin{restatable}{proposition}{equivalenceSWlog}
    \label{prop:equivalence_swlog}
    Let $\Tilde{\mu}, \Tilde{\nu} \in\mathcal{P}_p(S_d(\mathbb{R}))$, and $\Tilde{t}^A(B) = \mathrm{Tr}(A^T B)$ for $A,B\in S_d(\mathbb{R})$.
    We define
    \begin{equation}
        \mathrm{SymSW}_p^p(\Tilde{\mu}, \Tilde{\nu}) = \int_{S_d(\mathbb{R})} W_p^p(\Tilde{t}^A_\# \Tilde{\mu} ,\Tilde{t}^A_\# \Tilde{\nu} )\ \mathrm{d}\lambda_S(A) \enspace .
    \end{equation}
    Then, for $\mu,\nu\in\mathcal{P}_p(S_d^{++}(\mathbb{R}))$,
    \begin{equation}
        \mathrm{SPDSW}^p_p(\mu, \nu) = \mathrm{SymSW}_p^p(\log_\#\mu, \log_\#\nu) \enspace .
    \end{equation}
\end{restatable}
Thus, it seems natural to compare the results obtained with $\mathrm{SPDSW}$ to the Euclidean counterpart $\log\mathrm{SW} = \mathrm{SW}(\log_\# \cdot, \log_\# \cdot)$ where the distributions are made of projections in the $\log$ space and where the sampling is done with the uniform distribution on the sphere.
The Wasserstein distance is also well defined on Riemannian manifolds, and in particular on the space of SPD matrices.
Denoting $d$ a geodesic distance on $S_d^{++}(\mathbb{R})$, we can define the corresponding Wasserstein distance between $\mu,\nu\in\mathcal{P}_p(S_d^{++}(\mathbb{R}))$ as 
\begin{equation}
    W_p^p(\mu,\nu) = \inf_{\gamma\in\Pi(\mu,\nu)}\ \int d(X,Y)^p\ \mathrm{d}\gamma(X,Y) \enspace .
\end{equation}
In the following, we study properties of $\mathrm{SPDSW}$ and in particular, we show that it is a computationally efficient alternative to Wasserstein on $\mathcal{P}(S_d^{++}(\mathbb{R}))$ as it is topologically equivalent while having a better computational complexity and being better conditioned for regression of distributions. 

\subsection{Properties of $\mathrm{SPDSW}$}
\label{sec:properties}
We now derive theoretical properties of $\mathrm{SPDSW}$.

\paragraph{Topology.}

Following usual arguments which are valid for any sliced divergence with any projection, we can show that $\mathrm{SPDSW}$ is a pseudo-distance.
Here, $S_d^{++}(\mathbb{R})$ with the Log-Euclidean metric is of  null sectional curvature \citep{arsigny2005fast,xu2022unsupervised} and we have access to a diffeomorphism to a Euclidean space -- the $\log$ operator.
This allows us to show that $\mathrm{SPDSW}$ is a distance in \cref{prop:distance}.
\begin{restatable}{theorem}{distance}
    \label{prop:distance}
    Let $p\ge 1$, then $\mathrm{SPDSW}_p$ is a finite distance on $\mathcal{P}_p(S_d^{++}(\mathbb{R}))$.
\end{restatable}
In the case of the Affine-Invariant metric, the Riemannian manifold endowed with this metric has a non-positive and non-constant sectional curvature, and closed-forms of geodesics projections are not known to the best of our knowledge.
We can however derive Busemann coordinates, which involve a costly additional projection.
Moreover, whether or not it satisfies the indiscernible property remains an open question.
Hence, we focus on $\mathrm{SPDSW}$ with Log-Euclidean metric and discuss the use of the Affine-Invariant metric in \cref{sec:aispdsw}.

An important property which justifies the use of the SW distance in place of the Wasserstein distance in the Euclidean case is that they both metrize the weak convergence \citep{bonnotte2013unidimensional}.
We show in \cref{prop:weakcv} that this is also the case with $\mathrm{SPDSW}$ in $\mathcal{P}_p(S_d^{++}(\mathbb{R}))$. 
\begin{restatable}{theorem}{weakcv}
    \label{prop:weakcv}
    For $p\ge 1$, $\mathrm{SPDSW}_p$ metrizes the weak convergence, \emph{i.e.} for $\mu\in\mathcal{P}_p(S_d^{++}(\mathbb{R}))$ and a sequence $(\mu_k)_k$ in $\mathcal{P}_p(S_d^{++}(\mathbb{R}))$, $\lim_{k\to\infty}\mathrm{SPDSW}_p(\mu_k,\mu) = 0$ if and only if $(\mu_k)_k$ converges weakly to $\mu$.
\end{restatable}
Moreover, $\mathrm{SPDSW}_p$ and $W_p$ -- the $p$-Wasserstein distance with Log-Euclidean ground cost -- are also weakly equivalent on compactly supported measures on $\mathcal{P}_p(S_d^{++}(\mathbb{R}))$, as demonstrated in \cref{prop:bound}.
\begin{restatable}{theorem}{bound}
    \label{prop:bound}
    Let $p\ge 1$, let $\mu,\nu\in\mathcal{P}_p(S_d^{++}(\mathbb{R}))$.
    Then
    \begin{equation}
        \mathrm{SPDSW}_p^p(\mu,\nu) \le c_{d,p}^p W_p^p(\mu,\nu)\enspace ,
    \end{equation}
    where $c_{d,p}^p = \frac{1}{d}\int \|\theta\|_p^p\ \mathrm{d}\lambda(\theta)$.
    Let $R>0$ and $B(I,R)=\{A\in S_d^{++}(\mathbb{R}),\ d_{LE}(A,I_d)=\|\log A\|_F \le R\}$ be a closed ball. Then there exists a constant $C_{d,p,R}$ such that for all $\mu,\nu\in\mathcal{P}_p(B(I,R))$,
    \begin{equation}
        W_p^p(\mu,\nu) \le C_{d,p,R} \mathrm{SPDSW}_p(\mu,\nu)^{\frac{2}{d(d+1)+2}}\enspace .
    \end{equation}
\end{restatable}
\looseness=-1 The theorems above highlight that $\mathrm{SPDSW}_p$ behaves similarly to $W_p$ on $\mathcal{P}_p(S_d^{++}(\mathbb{R}))$.
Thus, it is justified to use $\mathrm{SPDSW}_p$ as a surrogate of Wasserstein and take advantage of the statistical and computational benefits that we present now.



\paragraph{Statistical properties.}
In practice, we approximate $\mathrm{SPDSW}$ using the plug-in estimator \citep{niles2022estimation, manole2022minimax}, \emph{i.e.} for $\mu,\nu\in\mathcal{P}_p(S_d^{++}(\mathbb{R}))$, we approximate $\mathrm{SPDSW}_p^p(\mu,\nu)$ by $\mathrm{SPDSW}_p^p(\hat{\mu}_n,\hat{\nu}_n)$ where $\hat{\mu}_n$ and $\hat{\nu}_n$ denote empirical distributions of $\mu$ and $\nu$.
Hence, we are interested in the speed of convergence towards $\mathrm{SPDSW}_p^p(\mu,\nu)$, which we call the sample complexity. We derive the convergence rate for $\mathrm{SPDSW}$ in \cref{prop:sample_complexity}, relying on the proof of \citet{nadjahi2020statistical} and on the sample complexity of the Wasserstein distance \citep{fournier2015rate}.
The sample complexity we find does not depend on the dimension, which is an important property of sliced divergences \citep{nadjahi2020statistical}.
\begin{restatable}{proposition}{sample}
    \label{prop:sample_complexity}
    Let $q > p\ge 1$, $\mu,\nu\in\mathcal{P}_p(S_d^{++}(\mathbb{R}))$, and $\hat{\mu}_n, \hat{\nu}_n$ the associated empirical measures.
    We define the moment of order $q$ by $M_q(\mu) = \int \|X\|_F^q\ \mathrm{d}\mu(X)$, and $M_{q}(\mu,\nu) = M_q(\log_\#\mu)^{1/q} + M_q(\log_\#\nu)^{1/q}$.
    Then, there exists a constant $C_{p,q}$ depending only on $p$ and $q$ such that
    \begin{align}
        \mathbb{E}\big[|\mathrm{SPDSW}_p(&\hat{\mu}_n,\hat{\nu}_n) - \mathrm{SPDSW}_p(\mu,\nu)|\big] \\
        \nonumber&\le \alpha_{n, p,q}C_{p,q}^{1/p} M_{q}(\mu,\nu)\enspace ,\\[.5em]\nonumber
        \text{where}\quad \alpha_{n,p,q} &= \begin{cases}
             n^{-1/(2p)} \ \text{ if } q>2p \\
            n^{-1/(2p)}\log(n)^{1/p} \ \text{ if } q=2p \\
            n^{-(q-p)/(pq)} \ \text{ if } q\in (p,2p) \enspace .
        \end{cases}\!
    \end{align}
\end{restatable}
\cref{prop:sample_complexity} assumes we can exactly compute the outer integral, which is not the case in practice, as it requires a Monte-Carlo approximation. In \cref{prop:projection_complexity}, we show that, $L$ being the number of projections, the Monte-Carlo error is $O(\frac{1}{\sqrt{L}})$ for a fixed dimension $d$.
This time, the dimension intervenes in $\mathrm{Var}_{A \sim \lambda_S}\left[W_p^p(t^A_\#\mu,t^A_\#\nu)\right]$.
\begin{restatable}{proposition}{projection}
    \label{prop:projection_complexity}
    Let $p\ge 1$, $\mu,\nu\in \mathcal{P}_{p}(S_d^{++}(\mathbb{R}))$. Then, the error made by the Monte Carlo estimate of $\mathrm{SPDSW}_p$ with L projections can be bounded as follows
    \begin{equation}
        \begin{aligned}
            &\mathbb{E}_A\left[|\widehat{\mathrm{SPDSW}}_{p,L}^p(\mu,\nu)-\mathrm{SPDSW}_p^p(\mu,\nu)|\right]^2 \\
            &\le \frac{1}{L}\mathrm{Var}_{A \sim \lambda_S}\left[W_p^p(t^A_\#\mu,t^A_\#\nu)\right]\enspace ,
        \end{aligned}
    \end{equation}
    where $\widehat{\mathrm{SPDSW}}_{p,L}^p(\mu,\nu) = \frac{1}{L}\sum_{i=1}^L W_p^p(t^{A_i}_\#\mu, t^{A_i}_\#\nu)$ with $(A_i)_{i=1}^L$ independent samples from $\lambda_S$.
\end{restatable}


\paragraph{Computational complexity and implementation.}

\looseness=-1 Let $\mu,\nu \in \mathcal{P}_p(S_d^{++}(\mathbb{R}))$ and $(X_i)_{i=1}^n$ (resp. $(Y_j)_{j=1}^m$) samples from $\mu$ (resp. from $\nu$).
We approximate $\mathrm{SPDSW}_p^p(\mu,\nu)$ by $\widehat{\mathrm{SPDSW}}_{p,L}^p(\hat{\mu}_n,\hat{\nu}_m)$ where $\hat{\mu}_n = \frac{1}{n}\sum_{i=1}^n \delta_{X_i}$ and $\hat{\nu}_m = \frac{1}{m}\sum_{j=1}^m \delta_{Y_j}$.
Sampling from $\lambda_O$ requires drawing a matrix $Z\in\mathbb{R}^{d\times d}$ with i.i.d normally distributed coefficients, and then taking the QR factorization with positive entries on the diagonal of $R$ \citep{mezzadri2006generate}, which needs $O(d^3)$ operations \citep[Section 5.2]{golub2013matrix}.
Then, computing $n$ matrix logarithms takes $O(nd^3)$ operations.
Given $L$ projections, the inner-products require $O(Lnd^2)$ operations, and the computation of the one-dimensional Wasserstein distances is done in $O(Ln\log n)$ operations.
Therefore, the complexity of $\mathrm{SPDSW}$ is $O(Ln(\log n + d^2) + (L+n)d^3)$.
The procedure is detailed in \cref{alg:spdsw}.
In practice, when it is required to call $\mathrm{SPDSW}$ several times in optimization procedures, the computational complexity can be reduced by drawing projections only once at the beginning.

Note that it is possible to draw symmetric matrices with complexity $O(d^2)$ by taking $A = \frac{Z + Z^T}{\|Z + Z^T\|_F}$.
Although this is a great advantage from the point of view of computation time, we leave it as an open question to know whether this breaks the bounds in \cref{prop:bound}.

\looseness=-1 We illustrate the computational complexity \emph{w.r.t} samples in \cref{fig:runtime}.
The computations have been performed on a GPU NVIDIA Tesla V100-DGXS 32GB using \texttt{PyTorch} \citep{paszke2017automatic}\footnote{
    Code is available at \url{https://github.com/clbonet/SPDSW}.
}.
We compare the runtime to the Wasserstein distance with Affine-Invariant (AIW) and Log-Euclidean (LEW) metrics, and to Sinkhorn algorithm (LES) which is a classical alternative to Wasserstein to reduce the computational cost.
When enough samples are available, then computing the Wasserstein distance takes more time than computing the cost matrix, and $\mathrm{SPDSW}$ is fast to compute.

\begin{algorithm}[tb]
   \caption{Computation of $\mathrm{SPDSW}$}
   \label{alg:spdsw}
    \begin{algorithmic}
       \STATE {\bfseries Input:} $(X_i)_{i=1}^n\sim \mu$, $(Y_j)_{j=1}^m\sim \nu$, $L$ the number of projections, $p$ the order
       \FOR{$\ell=1$ {\bfseries to} $L$}
       \STATE Draw $\theta\sim\mathrm{Unif}(S^{d-1})=\lambda$
       \STATE Draw $P\sim \mathrm{Unif}(O_d(\mathbb{R}))=\lambda_O$
       \STATE $A=P\mathrm{diag}(\theta)P^T$
       \STATE $\forall i,j,\ \hat{X}_i^{\ell}=t^A(X_i)$, $\hat{Y}_j^\ell=t^A(Y_j)$
       \STATE Compute $W_p^p(\frac{1}{n}\sum_{i=1}^n \delta_{\hat{X}_i^\ell}, \frac{1}{m}\sum_{j=1}^m \delta_{\hat{Y}_j^\ell})$
       \ENDFOR
       \STATE Return $\frac{1}{L}\sum_{\ell=1}^L W_p^p(\frac{1}{n}\sum_{i=1}^n \delta_{\hat{X}_i^\ell}, \frac{1}{m}\sum_{j=1}^m \delta_{\hat{Y}_j^\ell})$
    \end{algorithmic}
    \vspace{-5pt}
\end{algorithm}

\begin{figure}[t]
    \centering
    \includegraphics[width=\columnwidth]{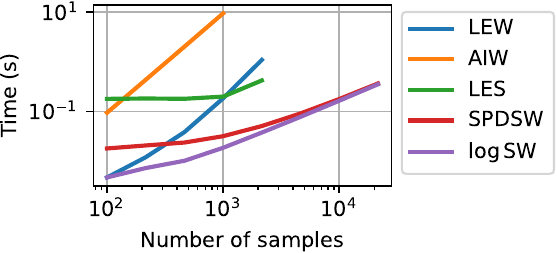}
    \caption{Runtime in log-log scale  of $\mathrm{SPDSW}$ and $\mathrm{\log SW}$ (200 proj., d=20) compared to alternatives based on Wasserstein between Wishart samples.
    Sliced discrepancies can scale to larger distributions in $S_d^{++}(\mathbb{R})$.
    }
    \label{fig:runtime}
    \vspace{-12pt}
\end{figure}

\section{From Brain Data to Distributions in $S_d^{++}(\mathbb{R})$}

M/EEG data consists of multivariate time series $X \in \mathbb{R}^{N_C \times T}$, with $N_C$ channels, and $T$ time samples.
A widely adopted model assumes that the measurements $X$ are linear combinations of $N_S$ sources $S \in \mathbb{R}^{N_S \times T}$ degraded by noise $N \in \mathbb{R}^{N_C \times T}$.
This leads to $X = AS + N$, where $A \in \mathbb{R}^{N_C \times N_S}$ is the forward linear operator \citep{hamalainen1993magnetoencephalography}.
A common practice in statistical learning on M/EEG data is to consider that the target is a function of the power of the sources, \emph{i.e.} $\mathbb{E}[SS^T]$ \citep{blankertz2007optimizing, dahne2014spoc, sabbagh2019manifold}.
In particular, a broad range of methods rely on second-order statistics of the measurements, \emph{i.e.} covariance matrices of the form $C = \frac{XX^T}{T}$, which are less costly and uncertain than solving the inverse problem to recover $S$ before training the model.
After proper rank reduction to turn the covariance estimates into SPD matrices \citep{harandi2017dimensionality}, and appropriate band-pass filtering to stick to specific physiological patterns \cite{blankertz2007optimizing}, Riemannian geometry becomes an appropriate tool to deal with such data.

In this section, we propose two applications of $\mathrm{SPDSW}$ to prediction tasks from M/EEG data.
More specifically, we introduce a new method to perform brain-age regression, building on the work of \citet{sabbagh2019manifold} and \citet{meunier2022distribution}, and another for domain adaptation in BCI.

\subsection{Distributions Regression for Brain-age Prediction}
\label{sec:brain_age}

\begin{figure*}[t]
    \includegraphics[width=\linewidth]{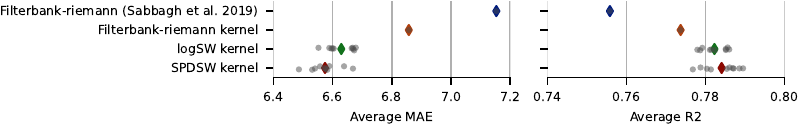}
    \caption{Average MAE and $R^2$ score for 10 random seeds on the Cam-CAN data-set with time-frames of 2s and 1000 projections.
    Kernel Ridge regression based on $\mathrm{SW}$ kernels performs best.
    $\mathrm{SPDSW}$ and $\log \mathrm{SW}$ are close to each other.
    Sampling from symmetric matrices offers a slight advantage but does not play a key role on performance.
    For information, Euclidean $\mathrm{SW}$ led to poor results on the task (MAE 9.7).
    }
    \vspace{-10pt}
    \label{fig:brain_age_average}
\end{figure*}

\looseness=-1 Learning to predict brain age from population-level neuroimaging data-sets can help characterize biological aging and disease severity \citep{spiegelhalter2016old, cole2017predicting, cole2018brain}.
Thus, this task has encountered more and more interest in the neuroscience community in recent years \citep{xifra2021estimating, peng2021accurate, engemann2022reusable}.
In particular, \citet{sabbagh2019manifold} take advantage of Riemannian geometry for feature engineering and prediction with the following steps.
First, one covariance estimate is computed per frequency band from each subject recording.
Then these covariance matrices are projected onto a lower dimensional space to make them full rank, for instance with a PCA.
Each newly obtained SPD matrix is projected onto the $\log$ space to obtain a feature after vectorization and aggregation among frequency bands.
Finally, a Ridge regression model predicts brain age.
This white-box method achieves state-of-the-art brain age prediction scores on MEG datasets like Cam-CAN \citep{taylor2017cambridge}.

\paragraph{MEG recordings as distributions of covariance matrices.}
Instead of modeling each frequency band by a unique covariance matrix, we propose to use a distribution of covariance matrices estimated from small time frames.
Concretely, given a time series $X \in \mathbb{R}^{N_C \times T}$ and a time-frame length $t < T$, a covariance matrix is estimated from each one of the $n = \lfloor \frac{T}{t} \rfloor$ chunks of signal available.
This process models each subject by as many empirical distributions of covariance estimates $(C_i)_{i=1}^n$ as there are frequency bands.
Then, all samples are projected on a lower dimensional space with a PCA, as done in \citet{sabbagh2019manifold}.
Here, we study whether modeling a subject by such distributions provides additional information compared to feature engineering based on a unique covariance matrix.
In order to perform brain age prediction from these distributions, we extend recent results on distribution regression with SW kernels \citep{kolouri2016sliced,meunier2022distribution} to SPD matrices, and show that $\mathrm{SPDSW}$ performs well on this prediction task while being easy to implement.

\paragraph{$\mathrm{SPDSW}$ kernels for distributions regression.}
As shown in \cref{sec:properties}, $\mathrm{SPDSW}$ is a well-defined distance on distributions in $S_d^{++}(\mathbb{R})$.
The most straightforward way to build a kernel from this distance is to resort to well-known Gaussian kernels, \emph{i.e.} $K(\mu, \nu) = e^{-\frac{1}{2\sigma^2}\mathrm{SPDSW}_2^2(\mu, \nu)}$.

However, this is not sufficient to make it a proper positive kernel.
Indeed, we need $\mathrm{SPDSW}$ to be a Hilbertian distance \citep{hein2005hilbertian}.
A pseudo-distance $d$ on $\mathcal{X}$ is Hilbertian if there exists a Hilbert space $\mathcal{H}$ and a feature map $\Phi : \mathcal{X} \rightarrow \mathcal{H}$ such that $ \forall x, y \in \mathcal{X}, d(x, y) = \|\Phi(x) - \Phi(y) \|_{\mathcal{H}}$.
We now extend \citet[Proposition 5]{meunier2022distribution} to the case of $\mathrm{SPDSW}$ in \cref{prop:hilbertian}.
\begin{restatable}{proposition}{hilbertianDistance}
    \label{prop:hilbertian}
    Let $m$ be the Lebesgue measure and let $\mathcal{H} = L^2([0,1] \times S_d(\mathbb{R}), m \otimes \lambda_S)$.
    We define $\Phi$ as
    \begin{equation}
        \begin{aligned}
        \Phi : \ &\mathcal{P}_2(S_d^{++}(\mathbb{R})) \rightarrow \mathcal{H}\\
        &\mu \mapsto \big( (q, A) \mapsto F^{-1}_{t^A_{\#}\mu}(q) \big) \enspace ,
        \end{aligned}
    \end{equation}
    where $F^{-1}_{t^A_{\#}\mu}$ is the quantile function of $t^A_{\#}\mu$.
    Then, $\mathrm{SPDSW}_2$ is Hilbertian and for all $\mu, \nu \in\mathcal{P}_2(S_d^{++}(\mathbb{R}))$,
    \begin{equation}
        \mathrm{SPDSW}_2^2(\mu, \nu) = \| \Phi(\mu) - \Phi(\nu) \|_{\mathcal{H}}^2 \enspace .
    \end{equation}
\end{restatable}
The proof is similar to the one of \citet{meunier2022distribution} for $\mathrm{SW}$ in Euclidean spaces and highlights two key results.
The first one is that $\mathrm{SPDSW}$ extensions of Gaussian kernels are valid positive definite kernels, as opposed to what we would get with the Wasserstein distance \citep{meunier2022distribution}.
The second one is that we have access to an explicit and easy-to-compute feature map that preserves $\mathrm{SPDSW}$, making it possible to avoid inefficient quadratic algorithms on empirical distributions from very large data.
In practice, we rely on the finite-dimensional approximation of projected distributions quantile functions proposed in \citet{meunier2022distribution} to compute the kernels more efficiently with the $\ell_2$-norm.
Then, we leverage Kernel Ridge regression for prediction \citep{murphy2012machine}.
Let $0 < q_1 < \dots < q_M < 1$, and $(A_1, \dots, A_L) \in S_d(\mathbb{R})^L$.
The approximate feature map has a closed-form expression in the case of empirical distributions and is defined as
\begin{equation}
    \hat{\Phi}(\mu) = \left(\frac{1}{\sqrt{ML}}F^{-1}_{t^{A_i}_{\#}\mu}(q_j)\right)_{1 \le j \le M, 1 \le i \le L} \enspace .
\end{equation}

Regarding brain-age prediction, we model each couple of subject $s$ and frequency band $f$ as an empirical distribution $\mu_n^{s,f}$ of covariance estimates $(C_i)_{i=1}^n$.
Hence, our data-set consists of the set of distributions in $S_d^{++}(\mathbb{R})$
\begin{equation}
    \left(\mu_n^{s,f} = \frac{1}{n} \sum_{i=1}^n \delta_{C_i}\right)_{s, f} \enspace .
\end{equation}
First, we compute the associated features $( \hat{\Phi}(\mu_n^{s, f}))_{s,f}$ by loading the data and band-pass filtering the signal once per subject.
Then, as we are interested in comparing each subject in specific frequency bands, we compute one approximate kernel matrix per frequency $f$, as follows
\begin{equation}
    K^f_{i,j} = e^{-\frac{1}{2\sigma^2}\| \hat{\Phi}(\mu_n^{i, f}) -  \hat{\Phi}(\mu_n^{j, f})\|^2_2} \enspace .
\end{equation}
Finally, the kernel matrix obtained as a sum over frequency bands, \emph{i.e.} $K = \sum_f K^f$, is plugged into the Kernel Ridge regression of \texttt{scikit-learn} \citep{pedregosa2011scikit}.

\paragraph{Numerical results}

We demonstrate the ability of our algorithm to perform well on brain-age prediction on the largest publicly available MEG data-set Cam-CAN \cite{taylor2017cambridge}, which contains recordings from 646 subjects at rest.
We take advantage of the benchmark provided by \citet{engemann2022reusable} -- available online\footnote{\url{https://github.com/meeg-ml-benchmarks/brain-age-benchmark-paper}} and described in \cref{subsec:brain_age_prediction_details} -- to replicate the same pre-processing and prediction steps from the data, and thus produce a meaningful and fair comparison.

For each one of the seven frequency bands, we divide every subject time series into frames of fixed length.
We estimate covariance matrices from each timeframe with OAS \citep{chen2010shrinkage} and apply PCA for rank-reduction, as in \citet{sabbagh2019manifold}, to obtain SPD matrices of size $53 \times 53$.
This leads to distributions of 275 points per subject and per frequency band.
In \citet{sabbagh2019manifold}, the authors rely on Ridge regression on vectorized projections of SPD matrices on the tangent space.
We also provide a comparison to Kernel Ridge regression based on a kernel with the Log-Euclidean metric, \emph{i.e.} $K_{i,j}^{\log} = e^{-\frac{1}{2\sigma^2}\| \log C_i - \log C_j \|^2_F}$.

\cref{fig:brain_age_average} shows that $\mathrm{SPDSW}$ and $\log \mathrm{SW}$ (1000 projections, time-frames of 2s) perform best in average on 10-folds cross-validation for 10 random seeds, compared to the baseline with Ridge regression \citep{sabbagh2019manifold} and to Kernel Ridge regression based on the Log-Euclidean metric, with identical pre-processing.
We provide more details on scores for each fold on a single random seed in \cref{fig:brain_age_boxplot}.
In particular, it seems that evaluating the distance between distributions of covariance estimates instead of just the average covariance brings more information to the model in this brain-age prediction task, and allows to improve the score.
Moreover, while $\mathrm{SPDSW}$ gives the best results, $\mathrm{logSW}$ actually performs well compared to baseline methods.
Thus, both methods seem to be usable in practice, even though sampling symmetric matrices and taking into account the Riemannian geometry improves the performances compared to $\mathrm{logSW}$.
Also note that Log-Euclidean Kernel Ridge regression works better than the baseline method based on Ridge regression \citep{sabbagh2019manifold}. 
Then, \cref{fig:variance_swspd} in the appendix shows that $\mathrm{SPDSW}$ does not suffer from variance with more than 500 projections in this use case with matrices of size 53 $\times$ 53.
Finally, \cref{fig:meg_timeframes} shows that there is a trade-off to find between smaller time-frames for more samples per distribution and larger time-frames for less noise in the covariance estimates and that this is an important hyper-parameter of the model.

\begin{table*}[t]
    \centering
    \caption{Accuracy and Runtime for Cross Session.}
    \small
    \resizebox{\linewidth}{!}{
        \begin{tabular}{ccccccccccccc}
             Subjects & Source & AISOTDA & & SPDSW & LogSW & LEW & LES & & SPDSW & LogSW & LEW & LES \\
             & & \citep{yair2019domain} & & \multicolumn{4}{c}{Transformations in $S_d^{++}(\mathbb{R})$} & & \multicolumn{4}{c}{Descent over particles}\\ \toprule
            1 & 82.21 & 80.90 & & 84.70 & 84.48 & 84.34 & 84.70 & & 85.20 & 85.20 & 77.94 & 82.92 \\
            3 & 79.85 & 87.86 & & 85.57 & 84.10 & 85.71 & 86.08 & & 87.11 & 86.37 & 82.42 & 81.47 \\
            7 & 72.20 & 82.29 & & 81.01 & 76.32 & 81.23 & 81.23 & & 81.81 & 81.73 & 79.06 & 73.29 \\
            8 & 79.34 & 83.25 & & 83.54 & 81.03 & 82.29 & 83.03 & & 84.13 & 83.32 & 80.07 & 85.02 \\
            9 & 75.76 & 80.25 & & 77.35 & 77.88 & 77.65 & 77.65 & & 80.30 & 79.02 & 76.14 & 70.45 \\
            \midrule 
            Avg. acc. & 77.87 & 82.93 & & 82.43 & 80.76 & 82.24 & 82.54 & & 83.71 & 83.12 & 79.13 & 78.63 \\
            Avg. time (s) & - & - & & \textbf{4.34} & \textbf{4.32} & 11.41 & 12.04 & & \textbf{3.68} & \textbf{3.67} & 8.50 & 11.43 \\
            \bottomrule
        \end{tabular}
    }
    \vspace{-5pt}
    \label{tab:cross_session}
\end{table*}

\subsection{Domain Adaptation for BCI}

\looseness=-1 BCI consists of establishing a communication interface between the brain and an external device, in order to assist or repair sensory-motor functions \citep{daly2008brain, nicolas2012brain, wolpaw2013brain}.
The interface should be able to correctly interpret M/EEG signals and link them to actions that the subject would like to perform.
One challenge of BCI is that ML methods are generally not robust to the change of data domain, which means that an algorithm trained on a particular subject will not be able to generalize to other subjects.
Domain adaptation (DA) 
\citep{bendavidDA} offers a solution to this problem by taking into account the distributional shift between source and target domains.
Classical DA techniques employed in BCI involve projecting target data on source data or vice versa, or learning a common embedding that erases the shift, sometimes with the help of optimal transport \citep{courty2016optimal}. 
As Riemannian geometry works well on BCI \citep{barachant2013classification,yger2016riemannian}, DA tools have been developed for SPD matrices \citep{yair2019domain, ju2022deep}.

\paragraph{$\mathrm{SPDSW}$ for domain adaptation on SPD matrices.}
We study two training frameworks on data from $\mathcal{P}(S_d^{++}(\mathbb{R}))$.
In the first case, a push forward operator $f_\theta$ is trained to change a distribution $\mu_S$ in the source domain into a distribution $\mu_T$ in the target domain by minimizing a loss of the form $L(\theta) = \mathcal{L}\big((f_\theta)_\#\mu_S, \mu_T\big)$, where $\mathcal{L}$ is a transport cost like Wasserstein on $\mathcal{P}(S_d^{++}(\mathbb{R}))$ or $\mathrm{SPDSW}$.
The model $f_\theta$ is a sequence of simple transformations in $S_d^{++}(\mathbb{R})$ \citep{rodrigues2018riemannian}, \emph{i.e.} $T_W(C) = W^T C W$ for $W \in S_d^{++}(\mathbb{R})$ (translations) or $W \in \mathrm{SO}_d$ (rotations), potentially combined to specific non-linearities \citep{huang2017riemannian}.
The advantage of such models is that they provide a high level of structure with a small number of parameters.

In the second case, we directly align the source on the target by minimizing $\mathcal{L}$ with a Riemannian gradient descent directly over the particles \citep{boumal2020introduction}, \emph{i.e.} by denoting $\mu_S ((x_i)_{i=1}^{|X_S|})=\frac{1}{|X_S|}\sum_{i=1}^{|X_S|}\delta_{x_i}$ with $X_S=\{x_i^S\}_i$ the samples of the source, we initialize at $(x_i^S)_{i=1}^{|X_S|}$ and minimize
$L((x_i)_{i=1}^{|X_S|}) = \mathcal{L}\left(\mu_S((x_i)_{i=1}^{|X_S|}), \mu_T\right)$.

We use \texttt{Geoopt} \citep{kochurov2020geoopt} and \texttt{Pytorch} \citep{paszke2017automatic} to optimize on manifolds.
Then, an SVM is trained on the vectorized projections of $X_S$ in the $\log$ space, \emph{i.e.} from couples $(\mathrm{vect}(\log x_i^S), y_i)_{i=1}^{|X_S|}$, and we evaluate the model on the target distribution.

\paragraph{Numerical results.}
In \cref{tab:cross_session}, we focus on cross-session classification for the BCI IV 2.a Competition dataset \citep{brunner2008bci} with 4 target classes and about 270 samples per subject and session.
We compare accuracies and runtimes for several methods run on a GPU Tesla V100-DGXS-32GB.
The distributions are aligned by minimizing different discrepancies, namely $\mathrm{SPDSW}$, $\mathrm{logSW}$, Log-Euclidean Wasserstein (LEW) and Sinkhorn (LES), computed with \texttt{POT}~\citep{flamary2021pot}.
Note that we did not tune hyper-parameters on each particular subject and discrepancy, but only used a grid search to train the SVM on the source data-set, and optimized each loss until convergence, \emph{i.e.} without early stopping.
We compare this approach to the naive one without DA, and to the barycentric OTDA \citep{courty2016optimal} with Affine-Invariant metric reported from \citet{yair2019domain}. 
We provide further comparisons on cross-subject in \cref{appendix:da}.
Our results show that all discrepancies give equivalent accuracies.
As expected, $\mathrm{SPDSW}$ has an advantage in terms of computation time compared to other transport losses.
Moreover, transformations in $S_d^{++}(\mathbb{R})$ and descent over the particles work almost equally well in the case of $\mathrm{SPDSW}$.
We illustrate the alignment we obtain by minimizing $\mathrm{SPDSW}$ in \cref{fig:pca_acc_projs}, with a PCA for visualization purposes.
Additionally, \cref{fig:pca_acc_projs} shows that $\mathrm{SPDSW}$ does not need too many projections to reach optimal performance.
We provide more experimental details in \cref{sec:exp_details}.

\begin{figure}[t]
    \centering
    \includegraphics[width=0.45\columnwidth]{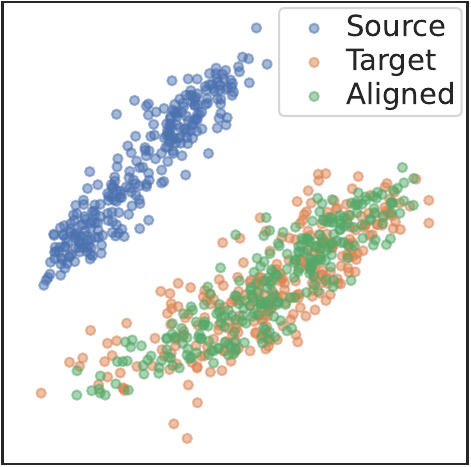}
    \includegraphics[width=0.47\columnwidth]{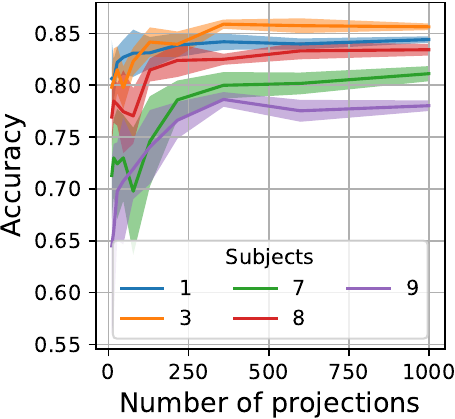}
    \caption{(\textbf{Left}) PCA on BCI data before and after alignment. Minimizing $\mathrm{SPDSW}$ with enough projections allows aligning sources on targets.
    (\textbf{Right}) Accuracy \emph{w.r.t} num. of projections for the cross-session task with transformations. Here, there is no need for too many projections to converge.}
    \label{fig:pca_acc_projs}
    \vspace{-15pt}
\end{figure}

\section{Conclusion}

\looseness=-1 We proposed $\mathrm{SPDSW}$, a discrepancy between distributions of SPD matrices with appealing properties such as being a distance and metrizing the weak convergence.
Being a Hilbertian metric, it can be plugged as is into Kernel methods, as we demonstrate for brain age prediction from MEG data.
Moreover, it is usable in loss functions dealing with distributions of SPD matrices, for instance in domain adaptation for BCI, with less computational complexity than its counterparts. Beyond M/EEG data, our discrepancy is of interest for any learning problem that involves distributions of SPD matrices, and we expect to see other applications of $\mathrm{SPDSW}$ in the future. One might also be interested in using other metrics on positive definite or semi-definite matrices such as the Bures-Wasserstein metric, with the additional challenges that this space is positively curved and not geodesically complete \citep{thanwerdas2021n}.


\section*{Acknowledgements}

Clément Bonet, Nicolas Courty and Lucas Drumetz contributions were supported by project DynaLearn from Labex CominLabs and Région Bretagne ARED DLearnMe. Nicolas Courty is partially funded by the project OTTOPIA ANR-20-CHIA-0030 of the French National Research Agency (ANR).
Benoît Malézieux contributions were supported by grants from Digiteo France.




\bibliography{references}

\begin{thebibliography}{81}
\providecommand{\natexlab}[1]{#1}
\providecommand{\url}[1]{\texttt{#1}}
\expandafter\ifx\csname urlstyle\endcsname\relax
  \providecommand{\doi}[1]{doi: #1}\else
  \providecommand{\doi}{doi: \begingroup \urlstyle{rm}\Url}\fi

\bibitem[Alvarez-Melis et~al.(2020)Alvarez-Melis, Mroueh, and
  Jaakkola]{alvarez2020unsupervised}
Alvarez-Melis, D., Mroueh, Y., and Jaakkola, T.
\newblock Unsupervised hierarchy matching with optimal transport over
  hyperbolic spaces.
\newblock In \emph{International Conference on Artificial Intelligence and
  Statistics}, pp.\  1606--1617. PMLR, 2020.

\bibitem[Arsigny et~al.(2005)Arsigny, Fillard, Pennec, and
  Ayache]{arsigny2005fast}
Arsigny, V., Fillard, P., Pennec, X., and Ayache, N.
\newblock \emph{Fast and Simple Computations on Tensors with Log-Euclidean
  Metrics.}
\newblock PhD thesis, INRIA, 2005.

\bibitem[Arsigny et~al.(2006)Arsigny, Fillard, Pennec, and
  Ayache]{arsigny2006log}
Arsigny, V., Fillard, P., Pennec, X., and Ayache, N.
\newblock Log-euclidean metrics for fast and simple calculus on diffusion
  tensors.
\newblock \emph{Magnetic Resonance in Medicine: An Official Journal of the
  International Society for Magnetic Resonance in Medicine}, 56\penalty0
  (2):\penalty0 411--421, 2006.

\bibitem[Barachant et~al.(2011)Barachant, Bonnet, Congedo, and
  Jutten]{barachant2011multiclass}
Barachant, A., Bonnet, S., Congedo, M., and Jutten, C.
\newblock Multiclass brain--computer interface classification by riemannian
  geometry.
\newblock \emph{IEEE Transactions on Biomedical Engineering}, 59\penalty0
  (4):\penalty0 920--928, 2011.

\bibitem[Barachant et~al.(2013)Barachant, Bonnet, Congedo, and
  Jutten]{barachant2013classification}
Barachant, A., Bonnet, S., Congedo, M., and Jutten, C.
\newblock Classification of covariance matrices using a riemannian-based kernel
  for bci applications.
\newblock \emph{Neurocomputing}, 112:\penalty0 172--178, 2013.

\bibitem[Bayraktar \& Guo(2021)Bayraktar and Guo]{bayraktar2021strong}
Bayraktar, E. and Guo, G.
\newblock Strong equivalence between metrics of wasserstein type.
\newblock \emph{Electronic Communications in Probability}, 26:\penalty0 1--13,
  2021.

\bibitem[Ben-David et~al.(2006)Ben-David, Blitzer, Crammer, and
  Pereira]{bendavidDA}
Ben-David, S., Blitzer, J., Crammer, K., and Pereira, F.
\newblock Analysis of representations for domain adaptation.
\newblock In \emph{Advances in Neural Information Processing Systems},
  volume~19. MIT Press, 2006.

\bibitem[Bhatia(2009)]{bhatia2009positive}
Bhatia, R.
\newblock Positive definite matrices.
\newblock In \emph{Positive Definite Matrices}. Princeton university press,
  2009.

\bibitem[Blankertz et~al.(2007)Blankertz, Tomioka, Lemm, Kawanabe, and
  Muller]{blankertz2007optimizing}
Blankertz, B., Tomioka, R., Lemm, S., Kawanabe, M., and Muller, K.-R.
\newblock Optimizing spatial filters for robust eeg single-trial analysis.
\newblock \emph{IEEE Signal processing magazine}, 25\penalty0 (1):\penalty0
  41--56, 2007.

\bibitem[Bogachev \& Ruas(2007)Bogachev and Ruas]{bogachev2007measure}
Bogachev, V.~I. and Ruas, M. A.~S.
\newblock \emph{Measure theory}, volume~1.
\newblock Springer, 2007.

\bibitem[Bonet et~al.(2022)Bonet, Chapel, Drumetz, and
  Courty]{bonet2022hyperbolic}
Bonet, C., Chapel, L., Drumetz, L., and Courty, N.
\newblock Hyperbolic sliced-wasserstein via geodesic and horospherical
  projections.
\newblock \emph{arXiv preprint arXiv:2211.10066}, 2022.

\bibitem[Bonet et~al.(2023)Bonet, Berg, Courty, Septier, Drumetz, and
  Pham]{bonet2022spherical}
Bonet, C., Berg, P., Courty, N., Septier, F., Drumetz, L., and Pham, M.-T.
\newblock Spherical sliced-wasserstein.
\newblock In \emph{International Conference on Learning Representations}, 2023.

\bibitem[Bonneel et~al.(2015)Bonneel, Rabin, Peyr{\'e}, and
  Pfister]{bonneel2015sliced}
Bonneel, N., Rabin, J., Peyr{\'e}, G., and Pfister, H.
\newblock Sliced and radon wasserstein barycenters of measures.
\newblock \emph{Journal of Mathematical Imaging and Vision}, 51\penalty0
  (1):\penalty0 22--45, 2015.

\bibitem[Bonnotte(2013)]{bonnotte2013unidimensional}
Bonnotte, N.
\newblock \emph{Unidimensional and evolution methods for optimal
  transportation}.
\newblock PhD thesis, Paris 11, 2013.

\bibitem[Boumal(2020)]{boumal2020introduction}
Boumal, N.
\newblock An introduction to optimization on smooth manifolds.
\newblock \emph{Available online, May}, 3, 2020.

\bibitem[Bridson \& Haefliger(2013)Bridson and Haefliger]{bridson2013metric}
Bridson, M.~R. and Haefliger, A.
\newblock \emph{Metric spaces of non-positive curvature}, volume 319.
\newblock Springer Science \& Business Media, 2013.

\bibitem[Brigant \& Puechmorel(2018)Brigant and Puechmorel]{brigant2018optimal}
Brigant, A.~L. and Puechmorel, S.
\newblock Optimal riemannian quantization with an application to air traffic
  analysis.
\newblock \emph{arXiv preprint arXiv:1806.07605}, 2018.

\bibitem[Brooks et~al.(2019)Brooks, Schwander, Barbaresco, Schneider, and
  Cord]{brooks2019riemannian}
Brooks, D., Schwander, O., Barbaresco, F., Schneider, J.-Y., and Cord, M.
\newblock Riemannian batch normalization for spd neural networks.
\newblock \emph{Advances in Neural Information Processing Systems}, 32, 2019.

\bibitem[Brunner et~al.(2008)Brunner, Leeb, M{\"u}ller-Putz, Schl{\"o}gl, and
  Pfurtscheller]{brunner2008bci}
Brunner, C., Leeb, R., M{\"u}ller-Putz, G., Schl{\"o}gl, A., and Pfurtscheller,
  G.
\newblock Bci competition 2008--graz data set a.
\newblock \emph{Institute for Knowledge Discovery (Laboratory of Brain-Computer
  Interfaces), Graz University of Technology}, 16:\penalty0 1--6, 2008.

\bibitem[Carriere et~al.(2017)Carriere, Cuturi, and Oudot]{carriere2017sliced}
Carriere, M., Cuturi, M., and Oudot, S.
\newblock Sliced wasserstein kernel for persistence diagrams.
\newblock In \emph{International conference on machine learning}, pp.\
  664--673. PMLR, 2017.

\bibitem[Chami et~al.(2021)Chami, Gu, Nguyen, and R{\'e}]{chami2021horopca}
Chami, I., Gu, A., Nguyen, D.~P., and R{\'e}, C.
\newblock Horopca: Hyperbolic dimensionality reduction via horospherical
  projections.
\newblock In \emph{International Conference on Machine Learning}, pp.\
  1419--1429. PMLR, 2021.

\bibitem[Chen et~al.(2010)Chen, Wiesel, Eldar, and Hero]{chen2010shrinkage}
Chen, Y., Wiesel, A., Eldar, Y.~C., and Hero, A.~O.
\newblock Shrinkage algorithms for mmse covariance estimation.
\newblock \emph{IEEE Transactions on Signal Processing}, 58\penalty0
  (10):\penalty0 5016--5029, 2010.

\bibitem[Chevallier et~al.(2017)Chevallier, Kalunga, and
  Angulo]{chevallier2017kernel}
Chevallier, E., Kalunga, E., and Angulo, J.
\newblock Kernel density estimation on spaces of gaussian distributions and
  symmetric positive definite matrices.
\newblock \emph{SIAM Journal on Imaging Sciences}, 10\penalty0 (1):\penalty0
  191--215, 2017.

\bibitem[Cole \& Franke(2017)Cole and Franke]{cole2017predicting}
Cole, J.~H. and Franke, K.
\newblock Predicting age using neuroimaging: innovative brain ageing
  biomarkers.
\newblock \emph{Trends in neurosciences}, 40\penalty0 (12):\penalty0 681--690,
  2017.

\bibitem[Cole et~al.(2018)Cole, Ritchie, Bastin, Hern{\'a}ndez,
  Mu{\~n}oz~Maniega, Royle, Corley, Pattie, Harris, Zhang,
  et~al.]{cole2018brain}
Cole, J.~H., Ritchie, S.~J., Bastin, M.~E., Hern{\'a}ndez, V.,
  Mu{\~n}oz~Maniega, S., Royle, N., Corley, J., Pattie, A., Harris, S.~E.,
  Zhang, Q., et~al.
\newblock Brain age predicts mortality.
\newblock \emph{Molecular psychiatry}, 23\penalty0 (5):\penalty0 1385--1392,
  2018.

\bibitem[Courty et~al.(2016)Courty, Flamary, Tuia, and
  Rakotomamonjy]{courty2016optimal}
Courty, N., Flamary, R., Tuia, D., and Rakotomamonjy, A.
\newblock Optimal transport for domain adaptation.
\newblock \emph{IEEE transactions on pattern analysis and machine
  intelligence}, 39\penalty0 (9):\penalty0 1853--1865, 2016.

\bibitem[Cui et~al.(2019)Cui, Qi, Wen, Lei, Li, Zhang, and
  Gu]{cui2019spherical}
Cui, L., Qi, X., Wen, C., Lei, N., Li, X., Zhang, M., and Gu, X.
\newblock Spherical optimal transportation.
\newblock \emph{Computer-Aided Design}, 115:\penalty0 181--193, 2019.

\bibitem[Cuturi(2013)]{cuturi2013sinkhorn}
Cuturi, M.
\newblock Sinkhorn distances: Lightspeed computation of optimal transport.
\newblock \emph{Advances in neural information processing systems}, 26, 2013.

\bibitem[D{\"a}hne et~al.(2014)D{\"a}hne, Meinecke, Haufe, H{\"o}hne,
  Tangermann, M{\"u}ller, and Nikulin]{dahne2014spoc}
D{\"a}hne, S., Meinecke, F.~C., Haufe, S., H{\"o}hne, J., Tangermann, M.,
  M{\"u}ller, K.-R., and Nikulin, V.~V.
\newblock Spoc: a novel framework for relating the amplitude of neuronal
  oscillations to behaviorally relevant parameters.
\newblock \emph{NeuroImage}, 86:\penalty0 111--122, 2014.

\bibitem[Daly \& Wolpaw(2008)Daly and Wolpaw]{daly2008brain}
Daly, J.~J. and Wolpaw, J.~R.
\newblock Brain--computer interfaces in neurological rehabilitation.
\newblock \emph{The Lancet Neurology}, 7\penalty0 (11):\penalty0 1032--1043,
  2008.

\bibitem[Engemann et~al.(2022)Engemann, Mellot, H{\"o}chenberger, Banville,
  Sabbagh, Gemein, Ball, and Gramfort]{engemann2022reusable}
Engemann, D.~A., Mellot, A., H{\"o}chenberger, R., Banville, H., Sabbagh, D.,
  Gemein, L., Ball, T., and Gramfort, A.
\newblock A reusable benchmark of brain-age prediction from m/eeg resting-state
  signals.
\newblock \emph{Neuroimage}, 262:\penalty0 119521, 2022.

\bibitem[Fatras et~al.(2020)Fatras, Zine, Flamary, Gribonval, and
  Courty]{fatras2019learning}
Fatras, K., Zine, Y., Flamary, R., Gribonval, R., and Courty, N.
\newblock Learning with minibatch wasserstein : asymptotic and gradient
  properties.
\newblock In Chiappa, S. and Calandra, R. (eds.), \emph{Proceedings of the
  Twenty Third International Conference on Artificial Intelligence and
  Statistics}, volume 108 of \emph{Proceedings of Machine Learning Research},
  pp.\  2131--2141. PMLR, 26--28 Aug 2020.

\bibitem[Flamary et~al.(2021)Flamary, Courty, Gramfort, Alaya, Boisbunon,
  Chambon, Chapel, Corenflos, Fatras, Fournier, et~al.]{flamary2021pot}
Flamary, R., Courty, N., Gramfort, A., Alaya, M.~Z., Boisbunon, A., Chambon,
  S., Chapel, L., Corenflos, A., Fatras, K., Fournier, N., et~al.
\newblock Pot: Python optimal transport.
\newblock \emph{J. Mach. Learn. Res.}, 22\penalty0 (78):\penalty0 1--8, 2021.

\bibitem[Fletcher et~al.(2009)Fletcher, Moeller, Phillips, and
  Venkatasubramanian]{fletcher2009computing}
Fletcher, P.~T., Moeller, J., Phillips, J.~M., and Venkatasubramanian, S.
\newblock Computing hulls and centerpoints in positive definite space.
\newblock \emph{arXiv preprint arXiv:0912.1580}, 2009.

\bibitem[Fletcher et~al.(2011)Fletcher, Moeller, Phillips, and
  Venkatasubramanian]{fletcher2011horoball}
Fletcher, P.~T., Moeller, J., Phillips, J.~M., and Venkatasubramanian, S.
\newblock Horoball hulls and extents in positive definite space.
\newblock In \emph{Workshop on Algorithms and Data Structures}, pp.\  386--398.
  Springer, 2011.

\bibitem[Fournier \& Guillin(2015)Fournier and Guillin]{fournier2015rate}
Fournier, N. and Guillin, A.
\newblock On the rate of convergence in wasserstein distance of the empirical
  measure.
\newblock \emph{Probability Theory and Related Fields}, 162\penalty0
  (3):\penalty0 707--738, 2015.

\bibitem[Gaur et~al.(2018)Gaur, Pachori, Wang, and Prasad]{gaur2018multi}
Gaur, P., Pachori, R.~B., Wang, H., and Prasad, G.
\newblock A multi-class eeg-based bci classification using multivariate
  empirical mode decomposition based filtering and riemannian geometry.
\newblock \emph{Expert Systems with Applications}, 95:\penalty0 201--211, 2018.

\bibitem[Golub \& Van~Loan(2013)Golub and Van~Loan]{golub2013matrix}
Golub, G.~H. and Van~Loan, C.~F.
\newblock \emph{Matrix computations}.
\newblock JHU press, 2013.

\bibitem[H{\"a}m{\"a}l{\"a}inen et~al.(1993)H{\"a}m{\"a}l{\"a}inen, Hari,
  Ilmoniemi, Knuutila, and Lounasmaa]{hamalainen1993magnetoencephalography}
H{\"a}m{\"a}l{\"a}inen, M., Hari, R., Ilmoniemi, R.~J., Knuutila, J., and
  Lounasmaa, O.~V.
\newblock Magnetoencephalography—theory, instrumentation, and applications to
  noninvasive studies of the working human brain.
\newblock \emph{Reviews of modern Physics}, 65\penalty0 (2):\penalty0 413,
  1993.

\bibitem[Harandi et~al.(2017)Harandi, Salzmann, and
  Hartley]{harandi2017dimensionality}
Harandi, M., Salzmann, M., and Hartley, R.
\newblock Dimensionality reduction on spd manifolds: The emergence of
  geometry-aware methods.
\newblock \emph{IEEE transactions on pattern analysis and machine
  intelligence}, 40\penalty0 (1):\penalty0 48--62, 2017.

\bibitem[Hein \& Bousquet(2005)Hein and Bousquet]{hein2005hilbertian}
Hein, M. and Bousquet, O.
\newblock Hilbertian metrics and positive definite kernels on probability
  measures.
\newblock In \emph{International Workshop on Artificial Intelligence and
  Statistics}, pp.\  136--143. PMLR, 2005.

\bibitem[Hersche et~al.(2018)Hersche, Rellstab, Schiavone, Cavigelli, Benini,
  and Rahimi]{hersche2018fast}
Hersche, M., Rellstab, T., Schiavone, P.~D., Cavigelli, L., Benini, L., and
  Rahimi, A.
\newblock Fast and accurate multiclass inference for mi-bcis using large
  multiscale temporal and spectral features.
\newblock In \emph{2018 26th European Signal Processing Conference (EUSIPCO)},
  pp.\  1690--1694. IEEE, 2018.

\bibitem[Huang \& Van~Gool(2017)Huang and Van~Gool]{huang2017riemannian}
Huang, Z. and Van~Gool, L.
\newblock A riemannian network for spd matrix learning.
\newblock In \emph{Thirty-first AAAI conference on artificial intelligence},
  2017.

\bibitem[Huang et~al.(2015)Huang, Wang, Shan, Li, and Chen]{huang2015log}
Huang, Z., Wang, R., Shan, S., Li, X., and Chen, X.
\newblock Log-euclidean metric learning on symmetric positive definite manifold
  with application to image set classification.
\newblock In \emph{International conference on machine learning}, pp.\
  720--729. PMLR, 2015.

\bibitem[Ilea et~al.(2018)Ilea, Bombrun, Said, and
  Berthoumieu]{ilea2018covariance}
Ilea, I., Bombrun, L., Said, S., and Berthoumieu, Y.
\newblock Covariance matrices encoding based on the log-euclidean and affine
  invariant riemannian metrics.
\newblock In \emph{Proceedings of the IEEE Conference on Computer Vision and
  Pattern Recognition Workshops}, pp.\  393--402, 2018.

\bibitem[Ju \& Guan(2022)Ju and Guan]{ju2022deep}
Ju, C. and Guan, C.
\newblock Deep optimal transport on spd manifolds for domain adaptation.
\newblock \emph{arXiv preprint arXiv:2201.05745}, 2022.

\bibitem[Kochurov et~al.(2020)Kochurov, Karimov, and
  Kozlukov]{kochurov2020geoopt}
Kochurov, M., Karimov, R., and Kozlukov, S.
\newblock Geoopt: Riemannian optimization in pytorch.
\newblock \emph{arXiv preprint arXiv:2005.02819}, 2020.

\bibitem[Kolouri et~al.(2016)Kolouri, Zou, and Rohde]{kolouri2016sliced}
Kolouri, S., Zou, Y., and Rohde, G.~K.
\newblock Sliced wasserstein kernels for probability distributions.
\newblock In \emph{Proceedings of the IEEE Conference on Computer Vision and
  Pattern Recognition}, pp.\  5258--5267, 2016.

\bibitem[Manole et~al.(2022)Manole, Balakrishnan, and
  Wasserman]{manole2022minimax}
Manole, T., Balakrishnan, S., and Wasserman, L.
\newblock Minimax confidence intervals for the sliced wasserstein distance.
\newblock \emph{Electronic Journal of Statistics}, 16\penalty0 (1):\penalty0
  2252--2345, 2022.

\bibitem[McCann(2001)]{mccann2001polar}
McCann, R.~J.
\newblock Polar factorization of maps on riemannian manifolds.
\newblock \emph{Geometric \& Functional Analysis GAFA}, 11\penalty0
  (3):\penalty0 589--608, 2001.

\bibitem[Meunier et~al.(2022)Meunier, Pontil, and
  Ciliberto]{meunier2022distribution}
Meunier, D., Pontil, M., and Ciliberto, C.
\newblock Distribution regression with sliced {W}asserstein kernels.
\newblock In Chaudhuri, K., Jegelka, S., Song, L., Szepesvari, C., Niu, G., and
  Sabato, S. (eds.), \emph{Proceedings of the 39th International Conference on
  Machine Learning}, volume 162 of \emph{Proceedings of Machine Learning
  Research}, pp.\  15501--15523. PMLR, 17--23 Jul 2022.

\bibitem[Mezzadri(2006)]{mezzadri2006generate}
Mezzadri, F.
\newblock How to generate random matrices from the classical compact groups.
\newblock \emph{arXiv preprint math-ph/0609050}, 2006.

\bibitem[Murphy(2012)]{murphy2012machine}
Murphy, K.~P.
\newblock \emph{Machine learning: a probabilistic perspective}.
\newblock MIT press, 2012.

\bibitem[Nadjahi et~al.(2019)Nadjahi, Durmus, Simsekli, and
  Badeau]{nadjahi2019asymptotic}
Nadjahi, K., Durmus, A., Simsekli, U., and Badeau, R.
\newblock Asymptotic guarantees for learning generative models with the
  sliced-wasserstein distance.
\newblock \emph{Advances in Neural Information Processing Systems}, 32, 2019.

\bibitem[Nadjahi et~al.(2020)Nadjahi, Durmus, Chizat, Kolouri, Shahrampour, and
  Simsekli]{nadjahi2020statistical}
Nadjahi, K., Durmus, A., Chizat, L., Kolouri, S., Shahrampour, S., and
  Simsekli, U.
\newblock Statistical and topological properties of sliced probability
  divergences.
\newblock \emph{Advances in Neural Information Processing Systems},
  33:\penalty0 20802--20812, 2020.

\bibitem[Nicolas-Alonso \& Gomez-Gil(2012)Nicolas-Alonso and
  Gomez-Gil]{nicolas2012brain}
Nicolas-Alonso, L.~F. and Gomez-Gil, J.
\newblock Brain computer interfaces, a review.
\newblock \emph{sensors}, 12\penalty0 (2):\penalty0 1211--1279, 2012.

\bibitem[Niles-Weed \& Rigollet(2022)Niles-Weed and
  Rigollet]{niles2022estimation}
Niles-Weed, J. and Rigollet, P.
\newblock Estimation of wasserstein distances in the spiked transport model.
\newblock \emph{Bernoulli}, 28\penalty0 (4):\penalty0 2663--2688, 2022.

\bibitem[Paszke et~al.(2017)Paszke, Gross, Chintala, Chanan, Yang, DeVito, Lin,
  Desmaison, Antiga, and Lerer]{paszke2017automatic}
Paszke, A., Gross, S., Chintala, S., Chanan, G., Yang, E., DeVito, Z., Lin, Z.,
  Desmaison, A., Antiga, L., and Lerer, A.
\newblock Automatic differentiation in pytorch.
\newblock 2017.

\bibitem[Paty \& Cuturi(2019)Paty and Cuturi]{paty2019subspace}
Paty, F.-P. and Cuturi, M.
\newblock Subspace robust wasserstein distances.
\newblock In \emph{International conference on machine learning}, pp.\
  5072--5081. PMLR, 2019.

\bibitem[Pedregosa et~al.(2011)Pedregosa, Varoquaux, Gramfort, Michel, Thirion,
  Grisel, Blondel, Prettenhofer, Weiss, Dubourg, et~al.]{pedregosa2011scikit}
Pedregosa, F., Varoquaux, G., Gramfort, A., Michel, V., Thirion, B., Grisel,
  O., Blondel, M., Prettenhofer, P., Weiss, R., Dubourg, V., et~al.
\newblock Scikit-learn: Machine learning in python.
\newblock \emph{the Journal of machine Learning research}, 12:\penalty0
  2825--2830, 2011.

\bibitem[Pele \& Werman(2009)Pele and Werman]{pele2009fast}
Pele, O. and Werman, M.
\newblock Fast and robust earth mover's distances.
\newblock In \emph{2009 IEEE 12th international conference on computer vision},
  pp.\  460--467. IEEE, 2009.

\bibitem[Peng et~al.(2021)Peng, Gong, Beckmann, Vedaldi, and
  Smith]{peng2021accurate}
Peng, H., Gong, W., Beckmann, C.~F., Vedaldi, A., and Smith, S.~M.
\newblock Accurate brain age prediction with lightweight deep neural networks.
\newblock \emph{Medical image analysis}, 68:\penalty0 101871, 2021.

\bibitem[Pennec(2020)]{pennec2020manifold}
Pennec, X.
\newblock Manifold-valued image processing with spd matrices.
\newblock In \emph{Riemannian geometric statistics in medical image analysis},
  pp.\  75--134. Elsevier, 2020.

\bibitem[Pennec et~al.(2006)Pennec, Fillard, and Ayache]{pennec2006riemannian}
Pennec, X., Fillard, P., and Ayache, N.
\newblock A riemannian framework for tensor computing.
\newblock \emph{International Journal of computer vision}, 66\penalty0
  (1):\penalty0 41--66, 2006.

\bibitem[Peyr{\'e} et~al.(2019)Peyr{\'e}, Cuturi,
  et~al.]{peyre2019computational}
Peyr{\'e}, G., Cuturi, M., et~al.
\newblock Computational optimal transport: With applications to data science.
\newblock \emph{Foundations and Trends{\textregistered} in Machine Learning},
  11\penalty0 (5-6):\penalty0 355--607, 2019.

\bibitem[Rabin et~al.(2011)Rabin, Peyr{\'e}, Delon, and
  Bernot]{rabin2011wasserstein}
Rabin, J., Peyr{\'e}, G., Delon, J., and Bernot, M.
\newblock Wasserstein barycenter and its application to texture mixing.
\newblock In \emph{International Conference on Scale Space and Variational
  Methods in Computer Vision}, pp.\  435--446. Springer, 2011.

\bibitem[Rakotomamonjy et~al.(2021)Rakotomamonjy, Alaya, Berar, and
  Gasso]{rakotomamonjy2021statistical}
Rakotomamonjy, A., Alaya, M.~Z., Berar, M., and Gasso, G.
\newblock Statistical and topological properties of gaussian smoothed sliced
  probability divergences.
\newblock \emph{arXiv preprint arXiv:2110.10524}, 2021.

\bibitem[Rivin(2007)]{rivin2007surface}
Rivin, I.
\newblock Surface area and other measures of ellipsoids.
\newblock \emph{Advances in Applied Mathematics}, 39\penalty0 (4):\penalty0
  409--427, 2007.

\bibitem[Rodrigues et~al.(2018)Rodrigues, Jutten, and
  Congedo]{rodrigues2018riemannian}
Rodrigues, P. L.~C., Jutten, C., and Congedo, M.
\newblock Riemannian procrustes analysis: transfer learning for brain--computer
  interfaces.
\newblock \emph{IEEE Transactions on Biomedical Engineering}, 66\penalty0
  (8):\penalty0 2390--2401, 2018.

\bibitem[Rustamov \& Majumdar(2020)Rustamov and
  Majumdar]{rustamov2020intrinsic}
Rustamov, R.~M. and Majumdar, S.
\newblock Intrinsic sliced wasserstein distances for comparing collections of
  probability distributions on manifolds and graphs.
\newblock \emph{arXiv preprint arXiv:2010.15285}, 2020.

\bibitem[Sabbagh et~al.(2019)Sabbagh, Ablin, Varoquaux, Gramfort, and
  Engemann]{sabbagh2019manifold}
Sabbagh, D., Ablin, P., Varoquaux, G., Gramfort, A., and Engemann, D.~A.
\newblock Manifold-regression to predict from meg/eeg brain signals without
  source modeling.
\newblock \emph{Advances in Neural Information Processing Systems}, 32, 2019.

\bibitem[Sabbagh et~al.(2020)Sabbagh, Ablin, Varoquaux, Gramfort, and
  Engemann]{sabbagh2020predictive}
Sabbagh, D., Ablin, P., Varoquaux, G., Gramfort, A., and Engemann, D.~A.
\newblock Predictive regression modeling with meg/eeg: from source power to
  signals and cognitive states.
\newblock \emph{NeuroImage}, 222:\penalty0 116893, 2020.

\bibitem[Spiegelhalter(2016)]{spiegelhalter2016old}
Spiegelhalter, D.
\newblock How old are you, really? communicating chronic risk through
  ‘effective age’of your body and organs.
\newblock \emph{BMC medical informatics and decision making}, 16\penalty0
  (1):\penalty0 1--6, 2016.

\bibitem[Taylor et~al.(2017)Taylor, Williams, Cusack, Auer, Shafto, Dixon,
  Tyler, Henson, et~al.]{taylor2017cambridge}
Taylor, J.~R., Williams, N., Cusack, R., Auer, T., Shafto, M.~A., Dixon, M.,
  Tyler, L.~K., Henson, R.~N., et~al.
\newblock The cambridge centre for ageing and neuroscience (cam-can) data
  repository: Structural and functional mri, meg, and cognitive data from a
  cross-sectional adult lifespan sample.
\newblock \emph{neuroimage}, 144:\penalty0 262--269, 2017.

\bibitem[Thanwerdas \& Pennec(2023)Thanwerdas and Pennec]{thanwerdas2021n}
Thanwerdas, Y. and Pennec, X.
\newblock O (n)-invariant riemannian metrics on spd matrices.
\newblock \emph{Linear Algebra and its Applications}, 661:\penalty0 163--201,
  2023.

\bibitem[Villani(2009)]{villani2009optimal}
Villani, C.
\newblock \emph{Optimal transport: old and new}, volume 338.
\newblock Springer, 2009.

\bibitem[Wolpaw(2013)]{wolpaw2013brain}
Wolpaw, J.~R.
\newblock Brain--computer interfaces.
\newblock In \emph{Handbook of Clinical Neurology}, volume 110, pp.\  67--74.
  Elsevier, 2013.

\bibitem[Xifra-Porxas et~al.(2021)Xifra-Porxas, Ghosh, Mitsis, and
  Boudrias]{xifra2021estimating}
Xifra-Porxas, A., Ghosh, A., Mitsis, G.~D., and Boudrias, M.-H.
\newblock Estimating brain age from structural mri and meg data: Insights from
  dimensionality reduction techniques.
\newblock \emph{NeuroImage}, 231:\penalty0 117822, 2021.

\bibitem[Xu(2022)]{xu2022unsupervised}
Xu, H.
\newblock Unsupervised manifold learning with polynomial mapping on symmetric
  positive definite matrices.
\newblock \emph{Information Sciences}, 609:\penalty0 215--227, 2022.

\bibitem[Yair et~al.(2019)Yair, Dietrich, Talmon, and
  Kevrekidis]{yair2019domain}
Yair, O., Dietrich, F., Talmon, R., and Kevrekidis, I.~G.
\newblock Domain adaptation with optimal transport on the manifold of spd
  matrices.
\newblock \emph{arXiv preprint arXiv:1906.00616}, 2019.

\bibitem[Yger et~al.(2016)Yger, Berar, and Lotte]{yger2016riemannian}
Yger, F., Berar, M., and Lotte, F.
\newblock Riemannian approaches in brain-computer interfaces: a review.
\newblock \emph{IEEE Transactions on Neural Systems and Rehabilitation
  Engineering}, 25\penalty0 (10):\penalty0 1753--1762, 2016.

\end{thebibliography}
\bibliographystyle{icml2023}

\newpage
\appendix
\onecolumn

\setcounter{figure}{0}
\renewcommand{\thefigure}{\Alph{figure}}

\section{Complementary experiments}
\label{sec:comp_exp}

\subsection{Brain Age Prediction}

\paragraph{Performance of SPDSW-based brain age regression on 10-folds cross validation for one random seed.}
In \cref{fig:brain_age_boxplot}, we display the Mean Absolute Error (MAE) and the $R^2$ coefficient on 10-folds cross validation with one random seed.
$\mathrm{SPDSW}$ is run with time-frames of 2s and 1000 projections.

\begin{figure}[h!]
    \includegraphics[width=\linewidth]{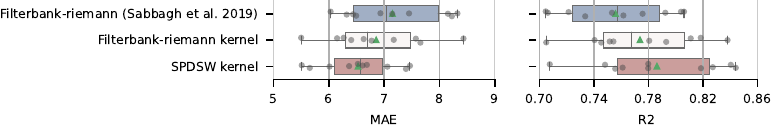}
    \caption{Results of 10-folds cross validation on the Cam-CAN data-set for one random seed.
    We display the Mean Absolute Error (MAE) and the $R^2$ coefficient.
    $\mathrm{SPDSW}$, with time-frames of 2s and 1000 projections, performs best.
    Note that Kernel Ridge regression based on the Log-Euclidean distance performs better than Ridge regression.}
    \label{fig:brain_age_boxplot}
\end{figure}

\paragraph{Performance of SPDSW-based brain age regression depending on number of projections.}
In \cref{fig:variance_swspd}, we display the MAE and $R^2$ score on brain age regression with different number of projections for 10 random seeds.
In this example, the variance and scores are acceptable for 500 projections and more. 

\begin{figure}[h!]
    \centering
    \includegraphics[width=\linewidth]{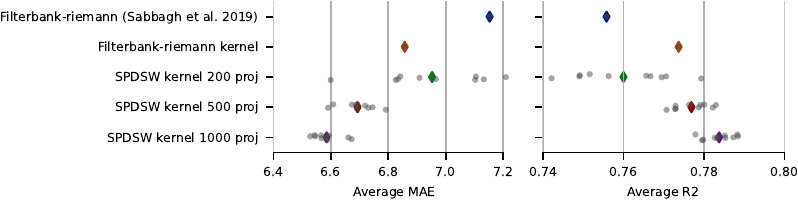}
    \caption{Average results for 10 random seeds with 200, 500 and 1000 projections for $\mathrm{SPDSW}$ compared to average MAE and $R^2$ obtained with Ridge and Kernel Ridge regression on features from covariance estimates \citep{sabbagh2019manifold}. With enough projections, $\mathrm{SPDSW}$ kernel does not suffer from variance and performs best.}
    \label{fig:variance_swspd}
\end{figure}

\paragraph{Performance of SPDSW-based brain age regression depending on timeframe length.}
In \cref{fig:meg_timeframes}, we display the MAE and $R^2$ score on brain age regression with different time-frame lengths for 10 random seeds.
The performance of $SPDSW$-kernel Ridge regression depends on a trade-off between the number of samples in each distribution (smaller time-frames for more samples), and the level of noise in the covariance estimate (larger time-frame for less noise).
In this example, time-frames of $400$ samples seems to be a good choice.

\begin{figure}[h!]
    \centering
    \includegraphics[width=\linewidth]{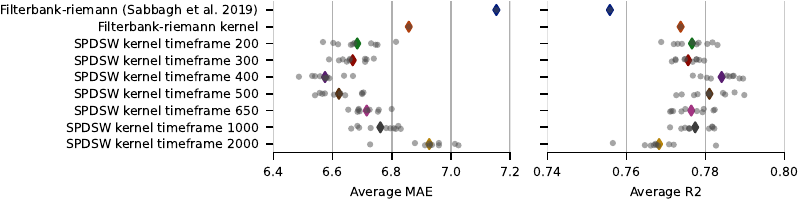}
    \caption{Average MAE and $R^2$ score on brain age regression with different time-frame lengths for 10 random seeds
    The performance depends on the time-frame length, and there is a trade-off to find between number of samples and noise in the samples.}
    \label{fig:meg_timeframes}
\end{figure}

\subsection{Domain Adaptation for BCI} \label{appendix:da}

\paragraph{Alignement.} We plot on \cref{fig:classes_cross_session} the classes of the target session (circles) and of the source session after alignment (crosses) on each subject. We observe that the classes seem to be well aligned, which explains why simple transformations work on this data-set. Hence, minimizing a discrepancy allows to align the classes even without taking them into account in the loss. More complicated data-sets might require to take into account the classes for the alignment.

\begin{figure}[h!]
    \centering
    \includegraphics[width=\columnwidth]{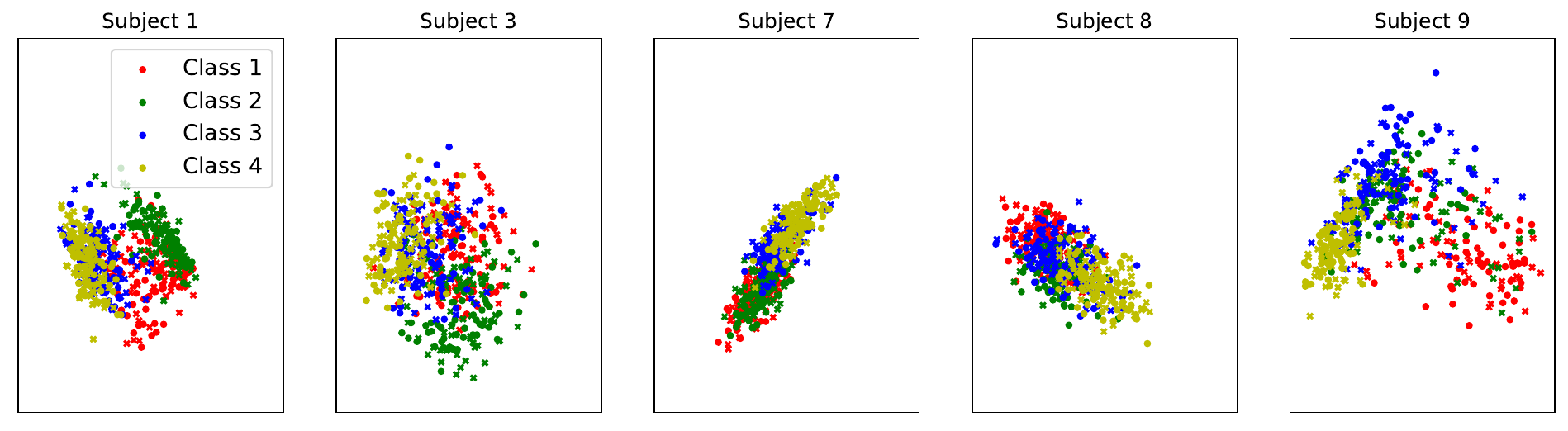}
    \caption{PCA representation on BCI data. Circles represent points from the target session and crosses points from the source after alignment.
    }
    \label{fig:classes_cross_session}
\end{figure}

\paragraph{Cross Subject Task.}

In \cref{tab:cross_subject}, we add the results obtained on the cross subject task. On the column ``subjects'', we denote the source subject, and we report in the table the mean of the accuracies obtained over all other subjects as targets.
The results for AISTODA are taken from \citet[Table 1.b, Alg.1 (u)]{yair2019domain}.
The preprocessing and hyperparameters might not be the same as in our setting.

\begin{table*}[t]
    \centering
    \caption{Accuracy and Runtime for Cross Subject.}
    \small
    \resizebox{\linewidth}{!}{
        \begin{tabular}{ccccccccccccc}
             Subjects & Source & AISOTDA & & SPDSW & LogSW & LEW & LES & & SPDSW & LogSW & LEW & LES \\
             & & \citep{yair2019domain} & & \multicolumn{4}{c}{Transformations in $S_d^{++}(\mathbb{R})$} & & \multicolumn{4}{c}{Descent over particles}\\ \toprule
            1 & 42.09 & 62.94 & & 61.91 & 60.50 & 62.89 & 63.64 & & 62.56 & 61.91 & 62.84 & 63.25 \\
            3 & 35.62 & 71.01 & & 66.40 & 66.53 & 66.34 & 66.30 & & 65.74 & 64.96 & 60.27 & 62.29 \\
            7 & 39.52 & 63.98 & & 60.42 & 57.29 & 60.89 & 60.43 & & 60.97 & 58.49 & 53.18 & 59.52 \\
            8 & 42.90 & 66.06 & & 61.09 & 60.19 & 61.29 & 62.14 & & 60.95 & 60.00 & 61.68 & 61.77 \\
            9 & 29.94 & 59.18 & & 53.31 & 50.63 & 54.79 & 54.89 & & 58.72 & 54.91 & 58.22 & 64.90\\
            \midrule 
            Avg. acc. & 38.01 & 64.43 & & 60.63 & 59.03 & 61.24 & 61.48 & & 61.79 & 60.05 & 59.24 & 62.55 \\
            Avg. time & - & - & & \textbf{4.34} & \textbf{4.31} & 11.76 & 11.21 & & \textbf{3.67} & \textbf{3.64} & 9.54 & 10.32 \\
            \bottomrule
        \end{tabular}
    }
    \label{tab:cross_subject}
\end{table*}

We add on Table \ref{tab:details_cross_subjects} 
the detailed accuracies between subjects (with on the rows the Table, and on the columns the targets) for SPDSW, LEW, 
and when applying the classifier on the source.

\begin{figure}[H]
    \centering
    \captionof{table}{Accuracy between subjects. The row denote the source and the columns the targets.}
    \label{tab:details_cross_subjects}
    \begin{minipage}{0.32\linewidth}
        \centering
        \captionof{table}{Source.}
        \small
        \resizebox{\columnwidth}{!}{
            \begin{tabular}{cccccc}
                & 1 & 3 & 7 & 8 & 9    \\ 
                \toprule
                1 & - & 52.22 & 50.55 & 39.02 & 26.58 \\
                3 & 34.43 & - & 30.10 & 49.62 & 27.43 \\
                7 & 52.01 & 53.33 & - & 26.14 & 26.58 \\
                8 & 49.82 & 57.78 & 24.35 & 0 & 39.66 \\
                9 & 26.74 & 28.52 & 24.72 & 39.39 & - \\
                \bottomrule
            \end{tabular}
        }
        \label{tab:cross_subject_src}
    \end{minipage}
    \hfill
    \begin{minipage}{0.32\linewidth}
        \centering
        \captionof{table}{Particles + $\mathrm{SPDSW}$.}
        \small
        \resizebox{\columnwidth}{!}{
            \begin{tabular}{cccccc}
                & 1 & 3 & 7 & 8 & 9    \\ 
                \toprule
                1 & - & 69.04 & 60.89 & 68.18 & 52.15 \\
                3 & 66.23 & - & 70.18 & 70.83 & 55.70 \\
                7 & 58.02 & 71.04 & - & 61.82 & 53.00 \\
                8 & 57.73 & 70.44 & 58.16 & - & 57.47 \\
                9 & 55.24 & 61.85 & 52.10 & 65.68 & - \\
                \bottomrule
            \end{tabular}
        }
        \label{tab:cross_subject_particles_spdsw}
    \end{minipage}
    \hfill
    \begin{minipage}{0.32\linewidth}
        \centering
        \captionof{table}{Particles + LEW.}
        \small
        \resizebox{\columnwidth}{!}{
            \begin{tabular}{cccccc}
                & 1 & 3 & 7 & 8 & 9    \\ 
                \toprule
                1 & - & 72.59 & 55.42 & 69.32 & 54.01 \\
                3 & 63.37 & - & 61.99 & 62.12 & 53.59 \\
                7 & 50.18 & 62.96 & - & 48.11 & 51.48 \\
                8 & 61.54 & 74.07 & 53.87 & - & 57.22 \\
                9 & 48.35 & 63.33 & 57.20 & 64.02 & - \\
                \bottomrule
            \end{tabular}
        }
        \label{tab:cross_subject_particles_lew}
    \end{minipage}
    
    \vspace{15pt}
    
    \begin{minipage}{0.3\linewidth}
        \centering
        \small
        \resizebox{\columnwidth}{!}{
            \begin{tabular}{cccccc}
            \end{tabular}
        }
        \label{tab:cross_subject_aisotda}
    \end{minipage}
    \hfill
    \begin{minipage}{0.3\linewidth}
        \centering
        \captionof{table}{Transf. + $\mathrm{SPDSW}$.}
        \small
        \resizebox{\columnwidth}{!}{
            \begin{tabular}{cccccc}
                & 1 & 3 & 7 & 8 & 9    \\ 
                \toprule
                1 & - & 68.00 & 59.04 & 68.79 & 51.81 \\
                3 & 68.42 & - & 71.07 & 69.24 & 56.88 \\
                7 & 57.66 & 69.78 & - & 60.83 & 53.42 \\
                8 & 62.71 & 72.07 & 53.87 & - & 55.70 \\
                9 & 53.92 & 59.04 & 40.15 & 60.15 & - \\
                \bottomrule
            \end{tabular}
        }
        \label{tab:cross_subject_transfs_spdsw}
    \end{minipage}
    \hfill
    \begin{minipage}{0.3\linewidth}
        \centering
        \captionof{table}{Transf. + LEW.}
        \small
        \resizebox{\columnwidth}{!}{
            \begin{tabular}{cccccc}
                & 1 & 3 & 7 & 8 & 9    \\ 
                \toprule
                1 & - & 70.00 & 59.78 & 68.18 & 53.59 \\
                3 & 69.60 & - & 71.59 & 69.32 & 54.85 \\
                7 & 57.88 & 73.37 & - & 61.74 & 53.59 \\
                8 & 63.00 & 72.22 & 54.24 & - & 55.70 \\
                9 & 55.31 & 60.00 & 39.48 & 64.02 & - \\
                \bottomrule
            \end{tabular}
        }
        \label{tab:cross_subject_transfs_lew}
    \end{minipage}
\end{figure}

\paragraph{Evolution of the accuracy \emph{w.r.t} the number of projections.}

On \cref{fig:pca_acc_projs}, we plot the evolution of the accuracy obtained by learning transformations on $S_d^{++}(\mathbb{R})$ on the cross session task. We report on Figure \ref{fig:acc_projs} the plot for the other cases. We compared the results for $L\in \{10, 16, 27, 46, 77, 129, 215, 359, 599, 1000\}$ projections, which are evenly spaced in log scale. Other parameters are the same as in \cref{tab:cross_session} and are detailed in \cref{sec:exp_details_bci}. The results were averaged over 10 runs, and we report the standard deviation.

\begin{figure}[t]
    \centering
    \hspace*{\fill}
    \subfloat[Transformations on cross-subjects.]{\label{fig:acc_projs_transfs_subject}\includegraphics[width=0.3\columnwidth]{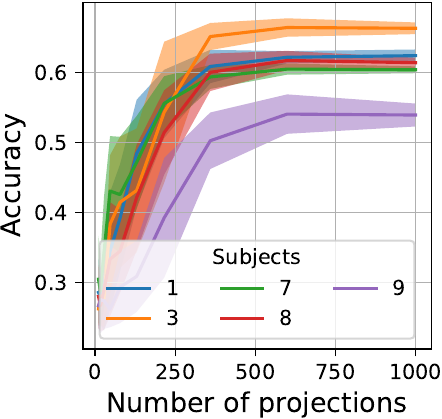}} \hfill
    \subfloat[Particles on cross-session.]{\label{fig:acc_projs_particles_session}\includegraphics[width=0.3\columnwidth]{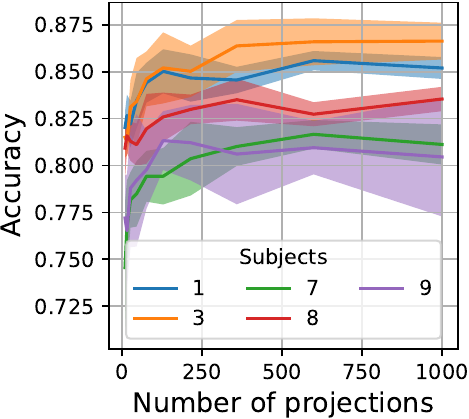}} \hfill
    \subfloat[Particles on cross-subject.]{\label{fig:acc_projs_particles_subject}\includegraphics[width=0.3\columnwidth]{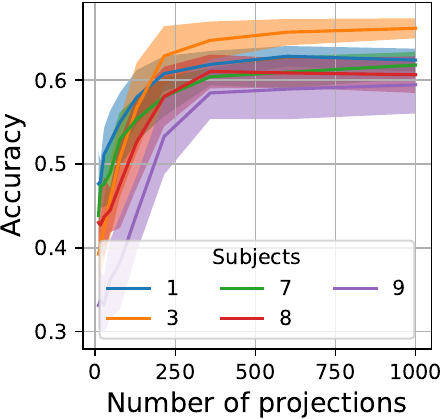}} \hfill
    \hspace*{\fill}
    \caption{Accuracy \emph{w.r.t} the number of projections when optimizing over particles or transformations, and for the cross-session task and cross subject task.
    In all cases, the accuracy converge for 500 projections.}
    \label{fig:acc_projs}
\end{figure}

\subsection{Illustrations}

\paragraph{Sample Complexity.} We illustrate \cref{prop:sample_complexity} in \cref{fig:sample_complexity} by plotting $\mathrm{SPDSW}$ and the Wasserstein distance with Log-Euclidean ground cost (LEW) between samples drawn from the same Wishart distribution, for $d=2$ and $d=50$. $\mathrm{SPDSW}$ is computed with $L=1000$ projections. We observe that $\mathrm{SPDSW}$ converges with the same speed in both dimensions while LEW converges slower in dimension 50.


\paragraph{Projection Complexity.} We illustrate \cref{prop:projection_complexity} on \cref{fig:proj_complexity} by plotting the absolute error between $\widehat{\mathrm{SPDSW}}_{2,L}^2$ and $\widehat{\mathrm{SPDSW}}_{2,L^*}^2$. We fix $L^*$ at 10000 which gives a good idea of the true value of $\mathrm{SPDSW}$ and we vary $L$ between $1$ and $10^3$ evenly in log scale. We average the results over 100 runs and plot 95\% confidence intervals. We observe that the Monte-Carlo error converges to 0 with a convergence rate of $O(\frac{1}{\sqrt{L}})$.


\begin{figure}[H]
    \centering
    \hspace*{\fill}
    \subfloat[Sample complexity of $D=\mathrm{SPDSW}$ and $D=\mathrm{LEW}$ for $d=2$ and $d=50$.]{\label{fig:sample_complexity}\includegraphics[width=0.48\columnwidth]{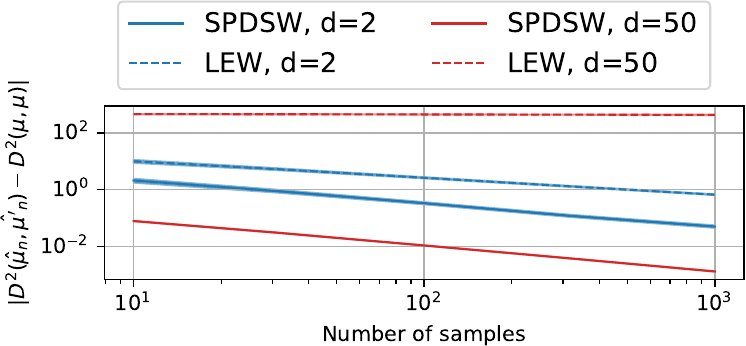}} \hfill
    \subfloat[Projection complexity of $\mathrm{SPDSW}$ and the $\mathrm{logSW}$ for $d=2$ and $d=20$.]{\label{fig:proj_complexity}\includegraphics[width=0.48\columnwidth]{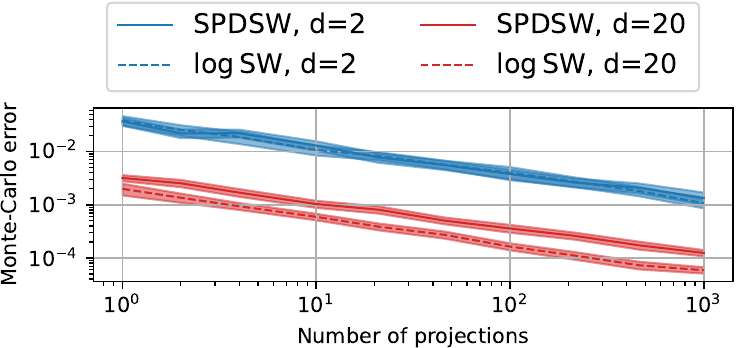}} \hfill
    \hspace*{\fill}
    \caption{Sample and projection complexity. Experiments are replicated 100 times and we report the 95\% confidence intervals. We note $\hat{\mu}_n$ and $\hat{\mu}'_n$ two different empirical distributions of $\mu$. The sample complexity of $\mathrm{SPDSW}$ does not depend on the dimension contrary to Wasserstein. The projections complexity has a slope which decreases in $O(\frac{1}{\sqrt{L}})$.}
    \label{fig:sample_proj_complexities}
\end{figure}

\section{Experimental details}
\label{sec:exp_details}

\subsection{Runtime}

In \cref{fig:runtime}, we plot the runtime \emph{w.r.t} the number of samples for different OT discrepancies. Namely, we compare $\mathrm{SPDSW}$, $\mathrm{\log SW}$, the Wasserstein distance with Affine-Invariant ground cost, the Wasserstein distance with Log-Euclidean ground cost, and the Sinkhorn algorithm used to compute the entropic regularized OT problem with Log-Euclidean ground cost. The distance ground costs are computed with \texttt{geoopt} \citep{kochurov2020geoopt} while Wasserstein and Sinkhorn are computed with \texttt{POT} \citep{flamary2021pot}. All computations are done on a A6000 GPU. We average the results over 20 runs and for $n\in\{100,215,464,1000,2154,4641,10000,21544,46415,100000\}$ samples, which are evenly spaced in log scale, from a Wishart distribution in dimension $d=20$. For the sliced methods, we fix $L=200$ projections. For the Sinkhorn algorithm, we use a stopping threshold of $10^{-10}$ with maximum $10^5$ iterations and a regularization parameter of $\epsilon = 1$.

\subsection{Brain Age Prediction}
\label{subsec:brain_age_prediction_details}

We reuse the code for preprocessing steps and benchmarking procedure described in \citet{engemann2022reusable} for the CamCAN data-set, and available at \url{https://github.com/meeg-ml-benchmarks/brain-age-benchmark-paper}, which we recall here.

The data consist of measurements from 102 magnetometers and 204 gradiometers.
First, we apply a band-pass filtering between 0.1Hz and 49Hz.
Then, the signal is subsampled with a decimation factor of 5, leading to a sample frequency of 200Hz.
Then, we apply the temporal signal-space-separation (tSSS).
Default settings were applied for the harmonic decomposition (8 components of the internal sources, 3 for the external sources) on a 10-s sliding window.
To discard segments for which inner and outer signal components were poorly distinguishable, we applied a correlation threshold of 98\%.

For analysis, the band frequencies used are the following: (0.1Hz, 1Hz), (1Hz, 4Hz), (4Hz, 8Hz), (8Hz, 15Hz), (15Hz, 26Hz), (26Hz, 35Hz), (35Hz, 49Hz).
The rank of the covariance matrices obtained after OAS is reduced to 53 with a PCA, which leads to the best score on this problem as mentioned in \citet{sabbagh2020predictive}.

The code for the MEG experiments is essentially based on the work by \citet{engemann2022reusable}, the class \texttt{SPDSW} available in the supplementary material, and the Kernel Ridge Regression of \texttt{scikit-learn}.
The full version will be added later in order to respect anonymity.

\subsection{Domain Adaptation for BCI} \label{sec:exp_details_bci}

For both the optimization over particles and over transformations, we use \texttt{geoopt} \citep{kochurov2020geoopt} with the Riemannian gradient descent. We now detail the hyperparameters and the procedure.

First, the data from the BCI Competition IV 2a are preprocessed using the code from \citet{hersche2018fast} available at \url{https://github.com/MultiScale-BCI/IV-2a}. We applied a band-pass filter between 8 and 30 Hz. With these hyper-parameters, we get one regularized covariance matrix per subject.

For all experiments, we report the results averaged over 5 runs. For the sliced discrepancies, we always use $L=500$ projections which we draw only once at the beginning. When optimizing over particles, we used a learning rate of $1000$ for the sliced methods and of $10$ for Wasserstein and Sinkhorn. The number of epochs was fixed at 500 for the cross-session task and for the cross-subject tasks. For the basic transformations, we always use 500 epochs and we choose a learning rate of $1e^{-1}$ on cross session and $5e^{-1}$ on cross subject for sliced methods, and of $1e^{-2}$ for Wasserstein and Sinkhorn. For the Sinkhorn algorithm, we use $\epsilon=10$ with the default hyperparameters from the \texttt{POT} implementation. Moreover, we only use one translation and rotation for the transformation.

Furthermore, the results reported for AISOTDA in \cref{tab:cross_session} and \cref{tab:cross_subject} are taken from \citet{yair2019domain} (Table 1.a, column Alg.1 (u)). We note however that they may not have used the same preprocessing and hyperparameters to load the covariance matrices.

\section{Proofs}
\label{sec:proofs}

\projGeodesic*

\begin{proof}
    Let $M \in S_d^{++}(\mathbb{R})$. We want to solve 
    \begin{equation}
        P^{\mathcal{G}_A}(M) = \argmin_{X \in \mathcal{G}_A}\ d_{LE}(X, M)^2\enspace .
    \end{equation}
    In the case of the Log-Euclidean metric, $\mathcal{G}_A = \{\exp(tA),\ t\in\mathbb{R}\}$.
    We have
    \begin{equation}
        \begin{aligned}
            d_{LE}(\exp(tA), M)^2 &= \|\log \exp(tA) - \log M\|_F^2\\
            &= \|tA - \log M \|_F^2\\
            &= t^2 \mathrm{Tr}(A^2) + \mathrm{Tr}(\log (M)^2) -2t\mathrm{Tr}(A\log M)\\
            &= g(t) \enspace .
        \end{aligned}
    \end{equation}
    Hence
    \begin{equation}
        g'(t) = 0 \iff t = \frac{\mathrm{Tr}(A\log M)}{\mathrm{Tr}(A^2)} \enspace .
    \end{equation}
    Therefore
    \begin{equation}
        P^{\mathcal{G}_A}(M) = \exp\left(\frac{\mathrm{Tr}(A\log M)}{\mathrm{Tr}(A^2)} A\right) = \exp\left(\mathrm{Tr}(A\log M) A\right) \enspace ,
    \end{equation}
    since $\|A\|_F^2 = \mathrm{Tr}(A^2) = 1$.
\end{proof}

\coordinateGeodesic*

\begin{proof}
    First, we give an orientation to the geodesic. This can be done by taking the sign of the inner product between $\log(P^{\mathcal{G}_A}(M))$ and $A$.
    \begin{equation}
        \begin{aligned}
            t^A(M) &= \mathrm{sign}(\langle A, \log(P^{\mathcal{G}_A}(M))\rangle_F) d\big(P^A(M), I\big) \\
            &= \mathrm{sign}(\langle A, \log(P^{\mathcal{G}_A}(M))\rangle_F) d\left(\exp\left( \mathrm{Tr}(A\log M) A\right), I\right) \\
            &= \mathrm{sign}(\langle A, \langle A, \log M\rangle_F A\rangle_F)  \|\langle A\log M\rangle_F A - \log I \|_F \\
            &= \mathrm{sign}(\langle A, \log M\rangle_F) |\langle A, \log M\rangle_F| \\
            &= \langle A, \log M\rangle_F \\
            &= \mathrm{Tr}(A\log M)\enspace .
        \end{aligned}
    \end{equation}
\end{proof}

There are actually two possible ways to find coordinates on geodesically complete Riemannian manifolds \citep{bonet2022hyperbolic}. The first one is to take the geodesic projection as previously done. A second solution is to use Busemann coordinates \citep{bridson2013metric, chami2021horopca}.

\begin{definition}
    Let $\gamma$ be a geodesic ray on a geodesically complete Riemannian manifold $\mathcal{M}$, \emph{i.e.} for all $s,t\ge 0$, $d(\gamma(s),\gamma(t)) = |t-s|$. Then, the Busemann function $B_\gamma$ associated to $\gamma$ is defined as, for all $x\in \mathcal{M}$,
    \begin{equation}
        B_\gamma(x) = \lim_{t\to\infty}\ \big(d(x,\gamma(t))-t\big)\enspace .
    \end{equation}
\end{definition}

This function allows to derive coordinates on geodesically complete geodesics. For example, on $\mathbb{R}^d$, it actually coincides with the geodesic projection as it can be shown that, for $\theta\in S^{d-1}$,
\begin{equation}
    \forall x\in \mathbb{R}^d,\ B_{\mathrm{span}(\theta)}(x) = -\langle x,\theta\rangle\enspace . 
\end{equation}

In \cref{prop:busemann_coords}, we derive a closed-form for the Busemann function associated to a geodesic ray on $S_d^{++}(\mathbb{R})$ passing through the identity.

\begin{proposition}[Busemann coordinates] \label{prop:busemann_coords}
    Let $A\in S_d(\mathbb{R})$ such that $\|A\|_F=1$, and let $\mathcal{G}_A$ be the associated geodesic line. Then, the Busemann function associated to $\mathcal{G}_A$ is defined as 
    \begin{equation}
        \forall M \in S_d^{++}(\mathbb{R}),\ B^A(M) = - \mathrm{Tr}(A\log M)\enspace .
    \end{equation}
\end{proposition}

\begin{proof}
    First, following \citep{bridson2013metric}, we have for all $M\in S_d^{++}(\mathbb{R}),$
    \begin{equation}
        B^A(M) = \lim_{t\to\infty}\ \big(d_{LE}(\gamma_A(t),M) - t \big) = \lim_{t\to\infty} \frac{d_{LE}(\gamma_A(t), M)^2 - t^2}{2t} \enspace ,
    \end{equation}
    denoting $\gamma_A:t\mapsto \exp(tA)$ is the geodesic line associated to $\mathcal{G}_A$. Then, we get
    \begin{equation}
        \begin{aligned}
            \frac{d_{LE}(\gamma_A(t), M)^2 - t^2}{2t} &= \frac{1}{2t}\left( \|\log \gamma_A(t) - \log M\|_F^2 - t^2\right)\\ 
            &= \frac{1}{2t}\left(\|tA-\log M\|_F^2 - t^2\right) \\
            &= \frac{1}{2t} \left(t^2 \|A\|_F^2 + \|\log M\|_F^2 - 2t\langle A, \log M\rangle_F - t^2\right) \\
            &= -\langle A,\log M\rangle_F + \frac{1}{2t}\|\log M\|_F^2 \enspace ,
        \end{aligned}
    \end{equation}
    using that $\|A\|_F = 1$. Then, by passing to the limit $t\to \infty$, we find
    \begin{equation}
        B^A(t) = -\langle A,\log M\rangle_F = -\mathrm{Tr}(A\log M)\enspace .
    \end{equation}
\end{proof}

We actually find that the Busemann coordinates are equal to the geodesic coordinates obtained in \cref{prop:geodesic_coordinate} up to the direction of the geodesic. We also show in \cref{prop:busemann_proj} that both projections on the geodesic coincide.

\begin{proposition}[Busemann projections] \label{prop:busemann_proj}
    Let $A\in S_d(\mathbb{R})$ with $\|A\|_F=1$ and let $\mathcal{G}_A$ the geodesic line associated. Then, for any $M\in S_d^{++}(\mathbb{R})$, the Busemann projection on $\mathcal{G}_A$ is 
    \begin{equation}
        P^A(M) = \exp\big(\mathrm{Tr}(A\log M)A\big)\enspace .
    \end{equation}
\end{proposition}

\begin{proof}
    The geodesic line is of the form 
        \begin{equation}
            \forall t\in \mathbb{R},\ \gamma_A(t) = \exp(tA) \enspace.
        \end{equation}
        
        We want to find a positive definite matrix on this geodesic with the same Busemann coordinate of $M$. Hence, we want to find $t$ such that
        \begin{equation}
            \begin{aligned}
                B^A(M) = B^A(\gamma_A(t)) &\iff \mathrm{Tr}\big(A\log M \big) = \mathrm{Tr}\big(A\log(\exp(tA))\big) \\
                &\iff \mathrm{Tr}(A\log M) = t\mathrm{Tr}(A^2) \\
                &\iff t = \frac{\mathrm{Tr}(A\log M)}{\mathrm{Tr}(A^2)} = \mathrm{Tr}(A\log M)\enspace ,
             \end{aligned}
        \end{equation}
        since $\|A\|_F^2 = \mathrm{Tr}(A^2) = 1$.
        
        Hence,
        \begin{equation}
            P^A(M) = \exp\left(\mathrm{Tr}(A\log M)A\right)\enspace .
        \end{equation}
\end{proof}

\uniformDistribution*

\proof{
    A matrix in $S_d(\mathbb{R})$ has a unique decomposition $P\mathrm{diag}(\theta)P^T$ up to permutations of the columns of $P \in \mathcal{O}_d$ and coefficients of $\theta \in S^{d-1}$.
    Thus, there is a bijection between $\{A \in S_d(\mathbb{R}),\ \| A\|_F = 1\}$ and the set $S_{\mathcal(O),S^{d-1}}$ of $d!$-tuple $\{ (P_1, \theta_1), \dots, (P_{d!}, \theta_{d!})  \in (\mathcal{O}_d \times S^{d-1})^{d!}\} $ such that $(P_i, \theta_i)$ is a permutation of $(P_j, \theta_j)$.
    Therefore, the uniform distribution $\lambda_{S_{\mathcal(O),S^{d-1}}}$ on $S_{\mathcal(O),S^{d-1}}$, defined as $\mathrm{d}\lambda_{S_{\mathcal(O),S^{d-1}}}((P_1, \theta_1), \dots, (P_{d!}, \theta_{d!})) = \sum_{i=1}^{n!} \mathrm{d}(\lambda_O \otimes \lambda) (P_i, \theta_i) = d!\cdot \mathrm{d}(\lambda_O \otimes \lambda) (P_1, \theta_1)$, allows to define a uniform distribution $\lambda_S$ on $\{A \in S_d(\mathbb{R}),\ \| A\|_F = 1\}$. Let $A = P \mathrm{diag}\theta P^T$ with $(P, \theta) \in \mathcal{O}_d \times S^{d-1}$, then
    \begin{equation}
       \mathrm{d}\lambda_S(A) = d!\ \mathrm{d}(\lambda_O \otimes \lambda) (P, \theta) \enspace . 
    \end{equation}

}

\equivalenceSWlog*

\proof{
    Denoting $\Tilde{t}^A(B)=\langle B,A\rangle_F$ for all $B\in S_d(\mathbb{R})$, we obtain using \citep[Lemma 6]{paty2019subspace}
    \begin{equation}
        \begin{aligned}
            W_p^p(\Tilde{t}^A_\#\log_\#\mu, \Tilde{t}^A_\#\log_\#\nu) &= \inf_{\gamma\in\Pi(\mu,\nu)}\ \int_{S_d^{++}(\mathbb{R})\times S_d^{++}(\mathbb{R})} |\Tilde{t}^A(\log(X))-\Tilde{t}^A(\log(Y))|^p\ \mathrm{d}\gamma(X,Y) \\
            &= \inf_{\gamma\in\Pi(\mu,\nu)}\ \int_{S_d^{++}(\mathbb{R})\times S_d^{++}(\mathbb{R})} |t^A(X)-t^A(Y)|^p\ \mathrm{d}\gamma(X,Y) \\
            &= W_p^p(t^A_\#\mu, t^A_\#\nu) \enspace ,
        \end{aligned}
    \end{equation}
    since $\Tilde{t}^A(\log X) = \langle A, \log X\rangle_F = t^A(X)$. Hence,
    \begin{equation}
        \mathrm{SymSW}_p^p(\log_\#\mu,\log_\#\nu) = \mathrm{SPDSW}_p^p(\mu,\nu)\enspace .
    \end{equation}
}

\distance*

\begin{proof}
    Let $p\ge 1$, and $\mu,\nu\in\mathcal{P}_p(S_d^{++}(\mathbb{R}))$. First, let's check that $\mathrm{SPDSW}_p^p(\mu,\nu)<\infty$.
    
    To see that, we will use on one hand \citet[Definition 6.4]{villani2009optimal} which states that on a Riemannian manifold $\mathcal{M}$, for any $x_0\in \mathcal{M}$,
    \begin{equation}
        \forall x,y\in \mathcal{M},\ d(x,y)^p \le 2^{p-1}\big(d(x,x^0)^p + d(x^0,y)^p\big)\enspace .
    \end{equation}
    Moreover, we will use that the projection $t^A$ is equal (up to a sign) to the Busemann function which is 1-Lipschitz \citep[II. Proposition 8.22]{bridson2013metric} and hence for any $A\in S_d(\mathbb{R})$ such that $\|A\|_F=1$ and $X,Y\in S_d^{++}(\mathbb{R})$, $|t^A(X)-t^A(Y)|\le d_{LE}(X,Y)$.
    Then, using \citet[Lemma 6]{paty2019subspace}, we have, for any $\pi\in\Pi(\mu,\nu)$ and $X_0\in S_d^{++}(\mathbb{R})$,
    \begin{equation}
        \begin{aligned}
            W_p^p(t^A_\#\mu,t^A_\#\nu) &= \inf_{\gamma\in \Pi(\mu,\nu)}\ \int_{S_d^{++}(\mathbb{R})\times S_d^{++}(\mathbb{R})} |t^A(X)-t^A(Y)|^p\ \mathrm{d}\gamma(X,Y) \\
            &\le \int_{S_d^{++}(\mathbb{R})\times S_d^{++}(\mathbb{R})} |t^A(X)-t^A(Y)|^p\ \mathrm{d}\pi(X,Y) \\
            &\le 2^{p-1} \left(\int_{S_d^{++}(\mathbb{R})} |t^A(X)-t^A(X_0)|^p \ \mathrm{d}\mu(X) + \int_{S_d^{++}(\mathbb{R})} |t^A(X_0)-t^A(Y)|^p\ \mathrm{d}\nu(Y)\right) \\
            &\le 2^{p-1} \left( \int_{S_d^{++}(\mathbb{R})} d_{LE}(X, X_0)^p \ \mathrm{d}\mu(X) + \int_{S_d^{++}(\mathbb{R})} d_{LE}(Y,X_0)^p\ \mathrm{d}\nu(Y) \right) \\
            &< \infty \enspace .
        \end{aligned}
    \end{equation}

    Let $p\ge 1$, then for all $\mu,\nu\in\mathcal{P}_p(S_d^{++}(\mathbb{R}))$, it is straightforward to see that $\mathrm{SPDSW}_p(\mu,\nu)\ge 0$, $\mathrm{SPDSW}_p(\mu,\nu)=\mathrm{SPDSW}_p(\nu,\mu)$. It is also easy to see that $\mu=\nu\implies \mathrm{SPDSW}_p(\mu,\nu)=0$ using that $W_p$ is a distance.
    
    Now, we can also derive the triangular inequality using the triangular inequality for $W_p$ and the Minkowski inequality:
    \begin{equation}
        \begin{aligned}
            \forall \mu,\nu,\alpha\in\mathcal{P}_p(S_d^{++}(\mathbb{R})),\ \mathrm{SPDSW}_p(\mu,\nu) &= \Big(\int_{S_d(\mathbb{R})} W_p^p(t^A\#\mu,t^A\#\nu)\ \mathrm{d}\lambda_S(A)\Big)^{\frac{1}{p}} \\
            &\le \Big(\int_{S_d(\mathbb{R})} \big(W_p(t^A_\#\mu,t^A_\#\alpha) + W_p(t^A_\#\alpha,t^A_\#\nu)\big)^p\ \mathrm{d}\lambda_S(A)\Big)^{\frac{1}{p}} \\
            &\le \Big(\int_{S_d(\mathbb{R})} W_p^p(t^A\#\mu,t^A_\#\alpha)\ \mathrm{d}\lambda_S(A)\Big)^{\frac{1}{p}} \\ &+ \Big(\int_{S_d(\mathbb{R})} W_p^p(t^A_\#\alpha, t^A_\#\nu)\ \mathrm{d}\lambda_S(A)\Big)^{\frac{1}{p}} \\
            &= \mathrm{SPDSW}_p(\mu,\alpha)+\mathrm{SPDSW}_p(\alpha,\nu)\enspace .
        \end{aligned}
    \end{equation}
    
    Lastly, we can derive the indiscernible property. Let $\mu,\nu\in\mathcal{P}_p(S_d^{++}(\mathbb{R}))$ such that $\mathrm{SPDSW}_p(\mu,\nu)=0$. Then, as for all $A\in S_d(\mathbb{R})$, $W_p^p(t^A_\#\mu,t^A_\#\nu)\ge 0$, it implies that for $\lambda_S$-almost every $A$, $W_p^p(t^A_\#\mu, t^A_\#\nu)=0$ which implies $t^A_\#\mu=t^A_\#\nu$ for $\lambda_S$-almost every $A$ since $W_p$ is a distance. By taking the Fourier transform, this implies that for all $s\in\mathbb{R}$, $\widehat{t^A_\#\mu}(s) = \widehat{t^A_\#\nu}(s)$. But, we have
    \begin{equation}
        \begin{aligned}
            \widehat{t^A_\#\mu}(s) &= \int_{\mathbb{R}} e^{-2i\pi ts}\ \mathrm{d}(t^A_\#\mu)(s) \\
            &= \int_{S_d^{++}(\mathbb{R})} e^{-2i\pi t^A(M) s}\ \mathrm{d}\mu(M) \\
            &= \int_{S_d^{++}(\mathbb{R})} e^{-2i\pi \langle sA, \log M\rangle_F}\ \mathrm{d}\mu(M) \\
            &= \int_{S_d(\mathbb{R})} e^{-2i\pi \langle sA, S\rangle_F}\ \mathrm{d}(\log_\#\mu)(S) \\
            &= \widehat{\log_\#\mu}(sA)\enspace .
        \end{aligned}
    \end{equation}
    Hence, we get that $\mathrm{SPDSW}_p(\mu,\nu)=0$ implies that for $\lambda_S$-almost every $A$,
    \begin{equation}
        \forall s\in \mathbb{R},\ \widehat{\log_\#\mu}(sA) = \widehat{t^A_\#\mu}(s) = \widehat{t^A_\#\nu}(s) = \widehat{\log_\#\nu}(sA)\enspace .
    \end{equation}
    By injectivity of the Fourier transform on $S_d(\mathbb{R})$, we get $\log_\#\mu=\log_\#\nu$. Then, as $\log$ is a bijection from $S_d^{++}(\mathbb{R})$ to $S_d(\mathbb{R})$, we have for all Borelian $M\subset S_d^{++}(\mathbb{R})$,
    \begin{equation}
        \begin{aligned}
            \mu(M) &= \int_{S_d^{++}(\mathbb{R})} \mathbb{1}_M(X)\ \mathrm{d}\mu(X) \\
            &= \int_{S_d(\mathbb{R})} \mathbb{1}_M(\exp(S))\ \mathrm{d}(\log_\#\mu)(S) \\
            &= \int_{S_d(\mathbb{R})} \mathbb{1}_M(\exp(S))\ \mathrm{d}(\log_\#\nu)(S) \\
            &= \int_{S_d^{++}(\mathbb{R})} \mathbb{1}_M(Y)\ \mathrm{d}\nu(Y) \\
            &= \nu(M)\enspace .
        \end{aligned}
    \end{equation}
    Hence, we conclude that $\mu=\nu$ and that $\mathrm{SPDSW}_p$ is a distance.
\end{proof}

To prove \cref{prop:weakcv}, we will adapt the proof of \citet{nadjahi2020statistical} to our projection. First, we start to adapt \citet[Lemma S1]{nadjahi2020statistical}:
\begin{lemma}[Lemma S1 in \citet{nadjahi2020statistical}] \label{lemma:nadjahi}
    Let $(\mu_k)_k \in \mathcal{P}_p(S_d^{++}(\mathbb{R}))$ and $\mu\in \mathcal{P}_p(S_d^{++}(\mathbb{R}))$ such that $\lim_{k\to\infty}\ \mathrm{SPDSW}_1(\mu_k,\mu)=0$. Then, there exists $\varphi:\mathbb{N}\to\mathbb{N}$ non decreasing such that $\mu_{\varphi(k)} \xrightarrow[k\to\infty]{\mathcal{L}} \mu$.
\end{lemma}

\begin{proof}
    By \citet[Theorem 2.2.5]{bogachev2007measure}, 
    \begin{equation}
        \lim_{k\to\infty}\ \int_{S^d(\mathbb{R})} W_1(t^A_\#\mu_k, t^A_\#\mu)\ \mathrm{d}\lambda_S(A) = 0
    \end{equation}
    implies that there exits a subsequence $(\mu_{\varphi(k)})_k$ such that for $\lambda_S$-almost every $A\in S_d(\mathbb{R})$,
    \begin{equation}
        W_1(t^A_\#\mu_{\varphi(k)}, t^A_\#\mu) \xrightarrow[k\to \infty]{}0\enspace .
    \end{equation}
    As the Wasserstein distance metrizes the weak convergence, this is equivalent to $t^A_\#\mu_{\varphi(k)} \xrightarrow[k\to\infty]{\mathcal{L}} t^A_\#\mu$.
    
    Then, by Levy's characterization theorem, this is equivalent with the pointwise convergence of the characterization function, \emph{i.e.} for all $t\in \mathbb{R}$,\ $\phi_{t^A_\#\mu_{\varphi(k)}}(t) \xrightarrow[k\to \infty]{} \phi_{t^A_\#\mu}(t)$.
    Moreover, we have for all $s\in \mathbb{R}$,
    \begin{equation}
        \begin{aligned}
            \phi_{t^A_\#\mu_{\varphi(k)}}(s) &= \int_\mathbb{R} e^{-its} \mathrm{d}(t^A_\#\mu_{\varphi(k)})(t) \\
            &= \int_{S_d^{++}(\mathbb{R})} e^{-i t^A(M) s}\ \mathrm{d}\mu_{\varphi(k)}(M) \\
            &= \int_{S_d^{++}(\mathbb{R})} e^{-i \langle sA, \log M\rangle_F}\ \mathrm{d}\mu_{\varphi(k)}(M) \\
            &= \int_{S_d(\mathbb{R})} e^{-i\langle sA, S\rangle_F} \ \mathrm{d}(\log_\#\mu_{\varphi(k)})(S) \\
            &= \phi_{\log_\#\mu_{\varphi(k)}}(sA) \\
            &\xrightarrow[k\to \infty]{} \phi_{\log_\#\mu}(sA) \enspace .
        \end{aligned}
    \end{equation}
    Then, working in $S_d(\mathbb{R})$ with the Frobenius norm, we can use the same proof of \citet{nadjahi2020statistical} by using a convolution with a gaussian kernel and show that it implies that $\log_\#\mu_{\varphi(k)}\xrightarrow[k\to\infty]{\mathcal{L}}\log_\#\mu$.
    
    Finally, let's show that it implies the weak convergence of $(\mu_{\varphi(k)})_k$ towards $\mu$. Let $f\in C_b(S_d^{++}(\mathbb{R}))$, then
    \begin{equation}
        \begin{aligned}
            \int_{S_d^{++}(\mathbb{R})} f \ \mathrm{d}\mu_{\varphi(k)} &= \int_{S_d(\mathbb{R})} f\circ \exp\ \mathrm{d}(\log_\#\mu_{\varphi(k)}) \\
            &\xrightarrow[k\to \infty]{} \int_{S_d(\mathbb{R})} f\circ \exp\ \mathrm{d}(\log_\#\mu) \\
            &= \int_{S_d^{++}(\mathbb{R})} f\ \mathrm{d}\mu \enspace .
        \end{aligned}
    \end{equation}
    Hence, we an conclude that $\mu_{\varphi(k)} \xrightarrow[k\to\infty]{\mathcal{L}} \mu$.
\end{proof}

\weakcv*

\begin{proof}
    First, we suppose that $\mu_k \xrightarrow[k\to\infty]{\mathcal{L}} \mu$ in $\mathcal{P}_p(S_d^{++}(\mathbb{R}))$. Then, by continuity, we have that for $\lambda_S$ almost every $A\in \mathcal{P}_p(S_d^{++}(\mathbb{R})$, $t^A_\#\mu_k \xrightarrow[k\to \infty]{} t^A_\#\mu$. Moreover, as the Wasserstein distance on $\mathbb{R}$ metrizes the weak convergence, $W_p(t^A_\#\mu_k,t^A_\#\mu) \xrightarrow[k\to\infty]{} 0$. Finally, as $W_p$ is bounded and it converges for $\lambda_S$-almost every $A$, we have by the Lebesgue convergence dominated theorem that $\mathrm{SPDSW}_p^p(\mu_k,\mu) \xrightarrow[k\to\infty]{} 0$.
        
    On the other hand, suppose that $\mathrm{SPDSW}_p(\mu_k,\mu)\xrightarrow[k\to\infty]{}0$. We first adapt Lemma S1 of \citep{nadjahi2020statistical} in Lemma \ref{lemma:nadjahi} and observe that by the Hölder inequality, 
    \begin{equation} \label{eq:holder}
        \mathrm{SPDSW}_1(\mu,\nu) \le \mathrm{SPDSW}_p(\mu,\nu) \enspace ,
    \end{equation}
    and hence $\mathrm{SPDSW}_1(\mu_k,\mu)\xrightarrow[k\to\infty]{} 0$.
    
    By the same contradiction argument as in \citet{nadjahi2020statistical}, let's suppose that $(\mu_k)_k$ does not converge to $\mu$. Then, by denoting $d_P$ the Lévy-Prokhorov metric, $\lim_{k\to\infty}d_P(\mu_k,\mu)\neq 0$. 
    
    Then, we have first that $\lim_{k\to\infty}\mathrm{SPDSW}_1(\mu_{\varphi(k)},\mu) = 0$. Thus, by Lemma \ref{lemma:nadjahi}, there exists a subsequence $(\mu_{\psi(\varphi(k))})_k$ such that $\mu_{\psi(\varphi(k))}\xrightarrow[k\to\infty]{\mathcal{L}}\mu$ which is equivalent to $\lim_{k\to\infty} d_P(\mu_{\psi(\varphi(k))}, \mu) = 0$ which contradicts the hypothesis.
    
    We conclude that $(\mu_k)_k$ converges weakly to $\mu$.
\end{proof}

For the proof of \cref{prop:bound}, we will first recall the following Theorem:
\begin{theorem}[\citep{rivin2007surface}, Theorem 3] \label{th3}
    Let $f:\mathbb{R}^d\mapsto\mathbb{R}$ a homogeneous function of degree $p$ (\emph{i.e.} $\forall \alpha\in\mathbb{R},\ f(\alpha x)=\alpha^p f(x)$). Then,
    \begin{equation}
        \Gamma\Big(\frac{d+p}{2}\Big)\int_{S^{d-1}} f(x)\ \lambda(\mathrm{d}x) = \Gamma\Big(\frac{d}{2}\Big)\mathbb{E}[f(X)] \enspace,
    \end{equation}
    where $\forall i\in\{1,...,d\}$, $X_i\sim\mathcal{N}(0,\frac12)$ and $(X_i)_i$ are independent.
\end{theorem}
Then, making extensive use of this theorem, we show the following lemma:
\begin{lemma} \label{lemma_eq}
    \begin{equation}
        \forall S\in S_d(\mathbb{R}),\ \int_{S^{d-1}} |\langle \mathrm{diag}(\theta),S\rangle_F|^p\ \lambda(\mathrm{d}\theta) = \frac{1}{d} \left(\sum_i S_{ii}^2\right)^{\frac{p}{2}} \int_{S^{d-1}} \|\theta\|_p^p\ \lambda(\mathrm{d}\theta) \enspace .
    \end{equation}
\end{lemma}

\begin{proof}
    Let $f:\theta\mapsto \|\theta\|_p^p=\sum_{i=1}^d \theta_i^p$, then we have $f(\alpha\theta)=\alpha^p f(\theta)$ and $f$ is $p$-homogeneous. By applying Theorem \ref{th3}, we have:
    \begin{equation}
        \begin{aligned}
            \int_{S^{d-1}} \|\theta\|_p^p\ \lambda(\mathrm{d}\theta) &= \frac{\Gamma\Big(\frac{d}{2}\Big)}{\Gamma\Big(\frac{d+p}{2}\Big)}\mathbb{E}[\|X\|_p^p] \text{ with $X_i \overset{\mathrm{iid}}{\sim}\mathcal{N}(0,\frac12)$}\\
            &= \frac{\Gamma\Big(\frac{d}{2}\Big)}{\Gamma\Big(\frac{d+p}{2}\Big)}d\ \mathbb{E}[|X_1|_p^p] \\
            &= \frac{\Gamma\Big(\frac{d}{2}\Big)}{\Gamma\Big(\frac{d+p}{2}\Big)}d\ \int |t|^p \frac{1}{\sqrt{\pi}} e^{-t^2} \mathrm{d}t \enspace .
        \end{aligned}
    \end{equation}
    
    On the other hand, let $\Tilde{f}:\theta\mapsto|\langle \mathrm{diag}(\theta),S\rangle_F|^p$, then $\Tilde{f}(\alpha\theta)=\alpha^p \Tilde{f}(\theta)$ and $\Tilde{f}$ is p-homogeneous. By applying Theorem \ref{th3}, we have:
    \begin{equation}
        \begin{aligned}
            \int_{S^{d-1}} |\langle \mathrm{diag}(\theta),S\rangle_F|^p\ \lambda(\mathrm{d}\theta) &= \frac{\Gamma\Big(\frac{d}{2}\Big)}{\Gamma\Big(\frac{d+p}{2}\Big)}\mathbb{E}[|\langle \mathrm{diag}(X),S\rangle_F|^p]\text{ with $X_i \overset{\mathrm{iid}}{\sim}\mathcal{N}(0,\frac12)$}\\
            &= \frac{\Gamma\Big(\frac{d}{2}\Big)}{\Gamma\Big(\frac{d+p}{2}\Big)} \int |t|^p \frac{1}{\sqrt{\sum_i S_{ii}^2 \pi}} e^{-\frac{t^2}{\sum_i z_{ii}^2}}\ \mathrm{d}t \text{ as $\langle \mathrm{diag}(X),S\rangle_F = \sum_i S_{ii} X_i \sim \mathcal{N}\Big(0,\frac{\sum_i S_{ii}^2}{2}\Big)$} \\
            &= \frac{\Gamma\Big(\frac{d}{2}\Big)}{\Gamma\Big(\frac{d+p}{2}\Big)} \left(\sum_i S_{ii}^2\right)^{\frac{p}{2}} \int |u|^p \frac{1}{\sqrt{\sum_i S_{ii}^2 \pi}} e^{-u^2} \sqrt{\sum_i S_{ii}^2} \mathrm{d}u \text{ by $u=\frac{t}{\sqrt{\sum_i S_{ii}^2}}$} \\
            &= \frac{\Gamma\Big(\frac{d}{2}\Big)}{\Gamma\Big(\frac{d+p}{2}\Big)} \left(\sum_i S_{ii}^2\right)^{\frac{p}{2}} \int |u|^p \frac{1}{\sqrt{\pi}} e^{-u^2} \mathrm{d}u \enspace .
        \end{aligned}
    \end{equation}
    
    Hence, we deduce that
    \begin{equation}
        \int_{S^{d-1}} |\langle \mathrm{diag}(\theta),S\rangle_F|^p\ \lambda(\mathrm{d}\theta) = \frac{1}{d} \left(\sum_i S_{ii}^2\right)^{\frac{p}{2}} \int_{S^{d-1}} \|\theta\|_p^p \ \mathrm{d}\lambda(\theta) \enspace .
    \end{equation}
\end{proof}

\bound*

\begin{proof}
    First, we show the upper bound of $\mathrm{SPDSW}_p$. Let $\mu,\nu\in\mathcal{P}_p(S_d^{++}(\mathbb{R})$ and $\gamma\in \Pi(\mu,\nu)$ an optimal coupling. Then, following the proof of \citet[Proposition 5.1.3]{bonnotte2013unidimensional}, and using \citet[Lemma 6]{paty2019subspace} combined with the fact that $(t^A\otimes t^A)_\#\gamma\in\Pi(t^A_\#\mu,t^A_\#\nu)$ for any $A\in S_d(\mathbb{R})$ such that $\|A\|_F=1$, we obtain
    \begin{equation} \label{eq:ineq_spdsw}
        \begin{aligned}
            \mathrm{SPDSW}_p^p(\mu,\nu) &= \int_{S_d(\mathbb{R})} W_p^p(t^A_\#\mu, t^A_\#\nu)\ \mathrm{d}\lambda_S(A) \\
            &\le \int_{S_d(\mathbb{R})} \int_{S_d^{++}(\mathbb{R})\times S_d^{++}(\mathbb{R})} |t^A(X)-t^A(Y)|^p\ \mathrm{d}\gamma(X,Y)\ \mathrm{d}\lambda_S(A) \\
            &= \int_{S_d(\mathbb{R})} \int_{S_d^{++}(\mathbb{R})\times S_d^{++}(\mathbb{R})} |\langle A, \log X - \log Y\rangle_F|^p\ \mathrm{d}\gamma(X,Y)\ \mathrm{d}\lambda_S(A) \\
            &= \int_{S^{d-1}}\int_{\mathcal{O}_d} \int_{S_d^{++}(\mathbb{R})\times S_d^{++}(\mathbb{R})} |\langle P\mathrm{diag}(\theta)P^T, \log X-\log Y\rangle_F|^p\ \mathrm{d}\gamma(X,Y)\ \mathrm{d}\lambda_O(P)\mathrm{d}\lambda(\theta) \\
            &= \int_{S^{d-1}}\int_{\mathcal{O}_d} \int_{S_d^{++}(\mathbb{R})\times S_d^{++}(\mathbb{R})} |\langle \mathrm{diag}(\theta), P^T(\log X-\log Y)P\rangle_F|^p\ \mathrm{d}\gamma(X,Y)\ \mathrm{d}\lambda_O(P)\mathrm{d}\lambda(\theta) \enspace .
        \end{aligned}
    \end{equation}
    By Lemma \ref{lemma_eq}, noting $S=P^T(\log X-\log Y)P$, we have
    \begin{equation}
        \begin{aligned}
            \int_{S^{d-1}} |\langle \mathrm{diag}(\theta), S\rangle_F|^p\ \mathrm{d}\lambda(\theta) &= \frac{1}{d} \left(\sum_i S_{ii}^2\right)^{\frac{p}{2}} \int_{S^{d-1}} \|\theta\|_p^p\ \mathrm{d}\lambda(\theta) \\
            &\le \frac{1}{d} \|S\|_F^p \int_{S^{d-1}} \|\theta\|_p^p\ \mathrm{d}\lambda(\theta) \enspace , 
        \end{aligned}
    \end{equation}
    since $\|S\|_F^2 = \sum_{i,j} S_{ij}^2 \ge \sum_i S_{ii}^2$. Moreover, $\|S\|_F = \|P^T(\log X-\log Y)P\|_F = \|\log X-\log Y\|_F$. Hence, coming back to \eqref{eq:ineq_spdsw}, we find
    \begin{equation} \label{eq:upperbound}
        \begin{aligned}
            \mathrm{SPDSW}_p^p(\mu,\nu) &\le \frac{1}{d} \int_{S^{d-1}} \|\theta\|_p^p \ \mathrm{d}\lambda(\theta) \int_{S_d^{++}(\mathbb{R})\times S_d^{++}(\mathbb{R})} \|\log X-\log Y\|_F^p\ \mathrm{d}\gamma(X,Y) \\
            &= \frac{1}{d}  \int_{S^{d-1}} \|\theta\|_p^p\ \mathrm{d}\lambda(\theta)\  W_p^p(\mu,\nu) \\
            &= c_{d,p}^p W_p^p(\mu,\nu) \enspace .
        \end{aligned}
    \end{equation}
    since $\gamma$ is an optimal coupling between $\mu$ and $\nu$ for the Wasserstein distance with Log-Euclidean cost.

    For the lower bound, let us first observe that 
    \begin{equation}
        \begin{aligned}
            W_1(\mu,\nu) &= \inf_{\gamma\in\Pi(\mu,\nu)}\ \int_{S_d^{++}(\mathbb{R})\times S_d^{++}(\mathbb{R})} d_{LE}(X,Y)\ \mathrm{d}\gamma(X,Y) \\
            &= \inf_{\gamma\in\Pi(\mu,\nu)}\ \int_{S_d^{++}(\mathbb{R})\times S_d^{++}(\mathbb{R})} \|\log X - \log Y\|_F\ \mathrm{d}\gamma(X,Y) \\
            &= \inf_{\gamma\in\Pi(\mu,\nu)}\ \int_{S_d(\mathbb{R})\times S_d(\mathbb{R})} \|U-V\|_F\ \mathrm{d}(\log\otimes\log)_\#\gamma(U,V) \\
            &= \inf_{\gamma\in\Pi(\log_\#\mu,\log_\#\nu)}\ \int_{S_d(\mathbb{R})\times S_d(\mathbb{R})} \|U-V\|_F\ \mathrm{d}\gamma(U,V) \\
            &= W_1(\log_\#\mu,\log_\#\nu) \enspace ,
        \end{aligned}
    \end{equation}
    where we used \citet[Lemma 6]{paty2019subspace}.
    
    
    Using \cref{prop:equivalence_swlog}, we have
    \begin{equation}
        \mathrm{SymSW}_1(\log_\#\mu,\log_\#\nu) = \mathrm{SPDSW}_1(\mu,\nu)\enspace .
    \end{equation}

    Therefore, as $S_d(\mathbb{R})$ is an Euclidean space of dimension $d(d+1)/2$, we can use \citep[Lemma 5.1.4]{bonnotte2013unidimensional} and we obtain that 
    \begin{equation}
        W_1(\log_\#\mu,\log_\#\nu) \le C_{d(d+1)/2} R^{d(d+1)/(d(d+1)+2)} \mathrm{SymSW}_1(\log_\#\mu,\log_\#\nu)^{2/(d(d+1)+2)} \enspace .
    \end{equation}
    Then, using that $\mathrm{SymSW}_1(\log_\#\mu,\log_\#\nu) = \mathrm{SPDSW}_1(\mu,\nu)$ and $W_1(\log_\#\mu,\log_\#\nu) = W_1(\mu,\nu)$, we obtain
    \begin{equation} \label{eq:ineq_w1}
        W_1(\mu,\nu) \le C_{d(d+1)/2} R^{d(d+1)/(d(d+1)+2)} \mathrm{SPDSW}_1(\mu,\nu)^{2/(d(d+1)+2)} \enspace .
    \end{equation}

    Now, following the proof of \citet[Theorem 5.1.5]{bonnotte2013unidimensional}, we use that on one hand, $W_p^p(\mu,\nu) \le (2R)^{p-1} W_1(\mu,\nu)$, and on the other hand, by Hölder, $\mathrm{SPDSW}_1(\mu,\nu)\le \mathrm{SPDSW}_p(\mu,\nu)$. Hence, using inequalities \eqref{eq:upperbound} and \eqref{eq:ineq_w1}, we get
    \begin{equation}
        \begin{aligned}
            \mathrm{SPDSW}_p^p(\mu,\nu) &\le c_{d,p}^p W_p^p(\mu,\nu) \\
            &\le (2R)^{p-1} W_1(\mu,\nu) \\
            &\le 2^{p-1} C_{d(d+1)/2} R^{p-1+d(d+1)/(d(d+1)+2)}\mathrm{SPDSW}_1(\mu,\nu)^{2/(d(d+1)/2)} \\
            &= C_{d,p}^d R^{p-2/(d(d+1))}\mathrm{SPDSW}_1(\mu,\nu)^{2/(d(d+1)+2)} \enspace .
        \end{aligned}
    \end{equation}
\end{proof}

\sample*

\begin{proof}
    In this proof, we will follow the derivations used in \citet{nadjahi2020statistical} and in \citep{rakotomamonjy2021statistical}. Notably, we will use the adaptation of \citet[Theorem 2]{fournier2015rate} reported in \citet[Lemma 1]{rakotomamonjy2021statistical}, which we recall now.
    \begin{lemma}[Lemma 1 in \citet{rakotomamonjy2021statistical} and Theorem 2 in \citet{fournier2015rate}] \label{lemma:fournier}
        Let $p\ge 1$ and $\eta\in\mathcal{P}_p(\mathbb{R})$. Denote $M_q(\eta)=\int |x|^q\ \mathrm{d}\eta(x)$ the moments of order $q$ and assume that $M_q(\eta)<\infty$ for some $q>p$. Then, there exists a constant $C_{p,q}$ depending only on $p,q$ such that for all $n\ge 1$,
        \begin{equation}
            \mathbb{E}[W_p^p(\hat{\eta}_n,\eta)] \le C_{p,q} M_q(\eta)^{p/q}\left(n^{-1/2}\mathbb{1}_{\{q>2p\}} + n^{-1/2}\log(n) \mathbb{1}_{\{q=2p\}} + n^{-(q-p)/q} \mathbb{1}_{\{q\in(p,2p)\}}\right) \enspace .
        \end{equation}
    \end{lemma}
    
    First, let us observe that by the triangular and reverse triangular inequalities, as well as Jensen for $x\mapsto x^{1/p}$ (which is concave since $p\ge 1$),
    \begin{equation}
        \begin{aligned}
            \mathbb{E}\left[|\mathrm{SPDSW}_p(\hat{\mu}_n,\hat{\nu}_n) - \mathrm{SPDSW}_p(\mu,\nu)|\right] &= \mathbb{E}[|\mathrm{SPDSW}_p(\hat{\mu}_n,\hat{\nu}_n)-\mathrm{SPDSW}_p(\hat{\mu}_n,\nu) \\ &+ \mathrm{SPDSW}_p(\hat{\mu}_n,\nu) - \mathrm{SPDSW}_p(\mu,\nu)|] \\
            &\le \mathbb{E}[|\mathrm{SPDSW}_p(\hat{\mu}_n,\hat{\nu}_n) - \mathrm{SPDSW}_p(\hat{\mu}_n,\nu)|] \\ &+ \mathbb{E}[|\mathrm{SPDSW}_p(\hat{\mu}_n,\nu)-\mathrm{SPDSW}_p(\mu,\nu)|] \\
            &\le \mathbb{E}[\mathrm{SPDSW}_p(\hat{\nu}_n,\nu)] + \mathbb{E}[\mathrm{SPDSW}_p(\hat{\mu}_n,\mu)] \\
            &\le \mathbb{E}[\mathrm{SPDSW}_p^p(\hat{\nu}_n,\nu)]^{1/p} + \mathbb{E}[\mathrm{SPDSW}_p^p(\hat{\mu}_n,\mu)]^{1/p} \enspace.
        \end{aligned}
    \end{equation}
    Moreover, by Fubini-Tonelli,
    \begin{equation}
        \begin{aligned}
            \mathbb{E}[\mathrm{SPDSW}_p^p(\hat{\mu}_n,\mu)] &= \mathbb{E}\left[\int_{S_d(\mathbb{R})} W_p^p(t^A_\#\hat{\mu}_n,t^A_\#\mu)\ \mathrm{d}\lambda_S(A)\right] \\
            &= \int_{S_d(\mathbb{R})} \mathbb{E}[W_p^p(t^A_\#\hat{\mu}_n,t^A_\#\mu)]\ \mathrm{d}\lambda_S(A)\enspace .
        \end{aligned}
    \end{equation}
    By applying Lemma \ref{lemma:fournier}, we get for $q>p$ that there exists a constant $C_{p,q}$ such that,
    \begin{equation}
        \mathbb{E}[W_p^p(t^A_\#\hat{\mu}_n,t^A_\#\mu)] \le C_{p,q} M_q(t^A_\#\mu)^{p/q}\left(n^{-1/2}\mathbb{1}_{\{q>2\}} + n^{-1/2}\log(n) \mathbb{1}_{\{q=2p\}} + n^{-(q-p)/q} \mathbb{1}_{\{q\in(p,2p)\}}\right) \enspace .
    \end{equation}
    Furthermore, using Cauchy-Schwartz and that $\|A\|_F=1$,
    \begin{equation}
        \begin{aligned}
            M_q(t^A_\#\mu) &= \int_{\mathbb{R}} |x|^q\ \mathrm{d}(t^A_\#\mu)(x) \\
            &= \int_{S_d^{++}(\mathbb{R})} |\langle A,\log X\rangle|^q\ \mathrm{d}\mu(X) \\
            &\le \int_{S_d^{++}(\mathbb{R})} \|\log X\|_F^q\ \mathrm{d}\mu(X) \\
            &= M_q(\log_\#\mu) \enspace .
        \end{aligned}
    \end{equation}
    Therefore, we have that
    \begin{equation}
        \begin{aligned}
            \mathbb{E}[\mathrm{SPDSW}_p^p(\hat{\mu}_n,\mu)] \le C_{p,q} M_q(\log_\#\mu)^{p/q} \left(n^{-1/2}\mathbb{1}_{\{q>2p\}} + n^{-1/2}\log(n) \mathbb{1}_{\{q=2p\}} + n^{-(q-p)/q} \mathbb{1}_{\{q\in(p,2p)\}}\right) \enspace,
        \end{aligned}
    \end{equation}
    and similarly
    \begin{equation}
        \begin{aligned}
            \mathbb{E}[\mathrm{SPDSW}_p^p(\hat{\nu}_n,\nu)] \le C_{p,q} M_q(\log_\#\nu)^{p/q} \left(n^{-1/2}\mathbb{1}_{\{q>2p\}} + n^{-1/2}\log(n) \mathbb{1}_{\{q=2p\}} + n^{-(q-p)/q} \mathbb{1}_{\{q\in(p,2p)\}}\right) \enspace.
        \end{aligned}
    \end{equation}
    Hence, we conclude that the sample complexity is 
    \begin{equation}
        \mathbb{E}\left[|\mathrm{SPDSW}_p(\hat{\mu}_n,\hat{\nu}_n) - \mathrm{SPDSW}_p(\mu,\nu)|\right] \le C_{p,q}^{1/p} \big(M_q(\log_\#\mu)^{1/q} + M_q(\log_\#\nu)^{1/q}\big) 
        \begin{cases}
            n^{-1/(2p)} \ \text{ if } q>2p \\
            n^{-1/(2p)}\log(n)^{1/p} \ \text{ if } q=2p \\
            n^{-(q-p)/(pq)} \ \text{ if } q\in (p,2p) \enspace .
        \end{cases}
    \end{equation}
\end{proof}

\projection*

\begin{proof}
    Let $(A_i)_{i=1}^L$ be iid samples of $\lambda_S$. Then, by first using Jensen inequality and then remembering that $\mathbb{E}_A[W_p^p(t^A_\#\mu,t^A_\#\nu)]=\mathrm{SPDSW}_p^p(\mu,\nu)$, we have
    \begin{equation}
        \begin{aligned}
            \mathbb{E}_A\left[|\widehat{\mathrm{SPDSW}}_{p,L}^p(\mu,\nu)-\mathrm{SPDSW}_p^p(\mu,\nu)|\right]^2 &\le \mathbb{E}_A\left[\left|\widehat{\mathrm{SPDSW}}_{p,L}^p(\mu,\nu)-\mathrm{SPDSW}_p^p(\mu,\nu)\right|^2\right]\\
            &= \mathbb{E}_A\left[\left|\frac{1}{L} \sum_{i=1}^L \big(W_p^p(t^{A_i}_\#\mu,t^{A_i}_\#\nu) - \mathrm{SPDSW}_p^p(\mu,\nu)\big)\right|^2\right] \\
            &= \frac{1}{L^2} \mathrm{Var}_A\left[\sum_{i=1}^L W_p^p(t^{A_i}_\#\mu, t^{A_i}_\#\nu)\right] \\
            &= \frac{1}{L} \mathrm{Var}_A\left[W_p^p(t^A_\#\mu, t^A_\#\nu)\right] \\
            &= \frac{1}{L} \int_{S_d(\mathbb{R})} \left(W_p^p(t^A_\#\mu, t^A_\#\nu)-\mathrm{SPDSW}_p^p(\mu,\nu)\right)^2\ \mathrm{d}\lambda_S(A) \enspace .
        \end{aligned}
    \end{equation}
\end{proof}

\hilbertianDistance*

\begin{proof}
    Let $\mu, \nu$ be probability distributions on $S_d^{++}(\mathbb{R})$ with moments of order $p=2$.
    Then
    \begin{align*}
        \mathrm{SPDSW}_2^2(\mu, \nu) &= \int_{S_d} \|F^{-1}_{t^A_{\#}\mu} - F^{-1}_{t^A_{\#}\nu}\|^2\ \mathrm{d}\lambda_S(A)\\
        &= \int_{S_d} \int_0^1 \big(F^{-1}_{t^A_{\#}\mu}(q) - F^{-1}_{t^A_{\#}\nu}(q) \big)^2\ \mathrm{d}q \mathrm{d}\lambda_S(A)\\
        &= \| \Phi(\mu) - \Phi(\nu) \|_{\mathcal{H}}^2 \enspace .
    \end{align*}
    Thus, $\mathrm{SPDSW}_2$ is Hilbertian.
\end{proof}

\section{SPDSW with Affine-Invariant Metric} \label{sec:aispdsw}

In the main part of the paper, we focused on $S_d^{++}(\mathbb{R})$ endowed with the Log-Euclidean metric. With this metric, $S_d^{++}(\mathbb{R})$ is a Riemannian manifold of constant null curvature as classical Euclidean spaces. Another metric of interest, very related to the Log-Euclidean one, is the Affine-Invariant metric, which yields a Riemannian manifold of non-constant and non-positive curvature \citep{bhatia2009positive, bridson2013metric}. The Log-Euclidean distance is actually a lower bound of the Affine-Invariant distance, and they coincide when the matrices commute. Notably, they share the same geodesics passing through the identity \citep[Section 3.6.1]{pennec2020manifold}. The Log-Euclidean metric can actually be seen as a good first order approximation of the Affine-Invariant metric \citep{arsigny2005fast,pennec2020manifold} which motivated the proposal of this metric. But it can lose some information when matrices are not commuting. Hence, we can wonder whether or not constructing a sliced discrepancy with projections obtained in $S_d^{++}(\mathbb{R})$ endowed with the Affine-Invariant metric could improve the results.

\subsection{Busemann Function on Affine-Invariant Space} \label{sec:busemann_ai}

As recalled in \cref{sec:bg_spd}, for the Affine-Invariant metric, the inner product in the tangent space is defined as
\begin{equation}
    \forall M\in S_d^{++}(\mathbb{R}),\ A,B\in \mathcal{T}_M,\ \langle A,B\rangle_M = \mathrm{Tr}(M^{-1}AM^{-1}B)\enspace ,
\end{equation}
and the corresponding geodesic distance is 
\begin{equation}
    \forall X, Y \in S_d^{++}(\mathbb{R}),\ d_{AI}(X,Y) = \sqrt{\mathrm{Tr}(\log(X^{-1}Y)^2)} \enspace .
\end{equation}
This distance notably satisfies the affine-invariant property, that is, for any $g\in GL_d(\mathbb{R})$, where $GL_d(\mathbb{R})$ denotes the set of invertibles matrices in $\mathbb{R}^{d\times d}$,
\begin{equation}
    \forall X,Y\in S_d^{++}(\mathbb{R}),\ d_{AI}(g\cdot X, g\cdot Y) = d_{AI}(X,Y) \enspace,
\end{equation}
where $g\cdot X = gXg^T$. Geodesics passing through the identity coincide with those of the Log-Euclidean metric, and are therefore of the form $\mathcal{G}_A= \{\exp(tA),\ t\in\mathbb{R}\}$ where $A\in S_d(\mathbb{R})$. Hence, we now need to find a projection of $M\in S_d^{++}(\mathbb{R})$ onto $\mathcal{G}_A$.

Unfortunately, to the best of our knowledge, there is no closed-form for the geodesic projection on geodesics. We will discuss here the horospherical projection which can be obtained with the Busemann function. For $A\in S_d(\mathbb{R})$ such that $\|A\|_F=1$, denoting $\gamma_A:t\mapsto \exp(tA)$ the geodesic passing through $I_d$ with direction $A$, the Busemann function $B^A$ associated to $\gamma_A$ writes as 
\begin{equation}
    \forall M\in S_d^{++}(\mathbb{R}),\ B^A(M) = \lim_{t\to\infty}\ \big(d_{AI}(\exp(tA),M)-t\big) \enspace.
\end{equation}
Contrary to the Log-Euclidean case, we cannot directly compute this quantity by expanding the distance since $\exp(-tA)$ and $M$ are not necessarily commuting. The main idea to solve this issue is to first find a group $G\subset GL_d(\mathbb{R})$ which will leave the Busemann function invariant. Then, we can find an element of this group which will project $M$ on the space of matrices commuting with $\exp(A)$. This part of the space is of null curvature, \emph{i.e.} it is isometric to an Euclidean space. In this case, we can compute the Busemann function as in \cref{prop:busemann_coords} as the matrices are commuting. Hence, the Busemann function is of the form
\begin{equation}
    B^A(M) = -\langle A, \log (\pi_A (M))\rangle_F \enspace , 
\end{equation}
where $\pi_A$ is a projection on the space of commuting matrices.

When $A$ is diagonal with sorted values such that $A_{11} > \dots > A_{dd}$, then the group leaving the Busemann function invariant is the set of upper triangular matrices with ones on the diagonal \citep[II. Proposition 10.66]{bridson2013metric}, \emph{i.e.} for such matrix $g$, $B^A(M) = B^A(gMg^T)$. If the points are sorted in increasing order, then the group is the set of lower triangular matrices. Let's note $G_U$ the set of upper triangular matrices with ones on the diagonal. For a general $A\in S_d(\mathbb{R})$, we can first find an appropriate diagonalization $A=P\Tilde{A}P^T$, where $\Tilde{A}$ is diagonal sorted, and apply the change of basis $\Tilde{M}=P^TMP$ \citep{fletcher2009computing}. Note that since we sample the eigenvalues from the uniform distribution on $S^{d-1}$, the values are all different almost surely. Therefore, we suppose that all the eigenvalues of $A$ have an order of multiplicity of one. By the affine-invariance property, the distances do not change, \emph{i.e.} $d_{AI}(\exp(tA),M) = d_{AI}(\exp(t\Tilde{A}),\Tilde{M})$ and hence, using the definition of the Busemann function, we have that $B^A(M) = B^{\Tilde{A}}(\Tilde{M})$. Then, we need to project $\Tilde{M}$ on the space of matrices commuting with $e^{\Tilde{A}}$ which we denote $F(A)$. By \citet[II. Proposition 10.67]{bridson2013metric}, this space corresponds to the diagonal matrices. Moreover, by \citet[II. Proposition 10.69]{bridson2013metric}, there is a unique pair $(g,D)\in G_U\times F(A)$ such that $\Tilde{M} = gDg^T$, and therefore, we can note $\pi_A(\Tilde{M})=D$. This decomposition actually corresponds to a UDU decomposition. If the eigenvalues of $A$ are sorted in increasing order, this would correspond to a LDL decomposition.

For more details about the Busemann function on the Affine-invariant space, we refer to \citet[Section II.10]{bridson2013metric} and \citet{fletcher2009computing, fletcher2011horoball}.

\subsection{Horospherical SPDSW}

Now that we know how to compute the coordinates on geodesics passing through the identity, we can derive an associated sliced discrepancy, which we call horospherical SPDSW ($\mathrm{HSPDSW}$) since the projection is made along level sets of the Busemann function, which are called horospheres \citep{fletcher2009computing}.

\begin{definition}
    \label{def:hswspd}
    Let $\lambda_O$ be the uniform distribution on orthogonal matrices $\mathcal{O}_d = \{P \in \mathbb{R}^{d \times d}, P^TP = PP^T = I\}$ (Haar distribution), $\lambda$ be the uniform distribution on $S^{d-1} = \{\theta \in \mathbb{R}^d, \|\theta\|_2=1\}$, and $\lambda_S$ be a probability distribution on $S_d(\mathbb{R})$ such that for all $V_S \in \sigma(S_d(\mathbb{R}))$, $\lambda_S(V_S) = (\lambda_O \otimes \lambda)(A_S)$ where $A_S = \{(P, \theta) \in \mathcal{O}_d \times S^{d-1},\ P \mathrm{diag}(\theta) P^T \in V_S\}$.
    Let $\mu, \nu \in \mathcal{P}_p(S_d^{++}(\mathbb{R}))$, the $\mathrm{HSPDSW}$ discrepancy is defined as
    \begin{equation}
        \mathrm{HSPDSW}_p^p(\mu,\nu) = \int_{S_d(\mathbb{R})} W_p^p(B^A_\#\mu, B^A_\#\nu)\ \mathrm{d}\lambda_S(A) \enspace ,
    \end{equation}
    where $B^A(M) = -\mathrm{Tr}\big(A \log (\pi_A (M))\big)$, with $\pi^A$ the projection derived in \cref{sec:busemann_ai}.
\end{definition}

On the side of theoretical properties, this discrepancy is still a pseudo-distance. However, since the projection $\log \circ \pi_A$ is not a diffeomorphism, whether the indiscernible property holds or not remains an open question.

On the computational side, it requires an additional projection step with a UDU decomposition for each sample and projection. Hence the overall complexity becomes $O(Ln(\log n + d^3))$ where the $O(Lnd^3)$ comes from the UDU decomposition. In practice, it takes more time than $\mathrm{SPDSW}$ for results which are pretty similar. In the same setting detailed in \cref{sec:exp_details}, we plot on \cref{fig:runtime_hpdsw} the runtime \emph{w.r.t} the number of samples and observe that it takes even more time than the Wasserstein distance. We detail the procedure to compute $\mathrm{HSPDSW}$ in \cref{alg:hspdsw}.

\begin{algorithm}[tb]
   \caption{Computation of $\mathrm{HSPDSW}$}
   \label{alg:hspdsw}
    \begin{algorithmic}
       \STATE {\bfseries Input:} $(X_i)_{i=1}^n\sim \mu$, $(Y_j)_{j=1}^m\sim \nu$, $L$ the number of projections, $p$ the order
       \FOR{$\ell=1$ {\bfseries to} $L$}
       \STATE Draw $\theta\sim\mathrm{Unif}(S^{d-1})=\lambda$
       \STATE Draw $P\sim \mathrm{Unif}(O_d(\mathbb{R}))=\lambda_O$
       \STATE Get $Q$ the permutation matrix such that $\Tilde{\theta} = Q\theta$ is sorted in decreasing order
       \STATE Set $A=\mathrm{diag}(\Tilde{\theta})$, $\Tilde{P}=PQ^T$
       \STATE $\forall i,j$, $\Tilde{X}_i^{\ell} = \Tilde{P}^T X_i \Tilde{P}$, $\Tilde{Y}_j^{\ell} = \Tilde{P}^T Y_j \Tilde{P}$
       \STATE $\forall i, j$, $D_i^\ell=UDU(\Tilde{X}_i^\ell)$, $\Delta_j^\ell = UDU(\Tilde{Y}_j^\ell)$
       \STATE  $\forall i,j,\ \hat{X}_i^{\ell}=t^A(D^\ell_i)$, $\hat{Y}_j^\ell=t^A(\Delta^\ell_j)$
       \STATE Compute $W_p^p(\frac{1}{n}\sum_{i=1}^n \delta_{\hat{X}_i^\ell}, \frac{1}{m}\sum_{j=1}^m \delta_{\hat{Y}_j^\ell})$
       \ENDFOR
       \STATE Return $\frac{1}{L}\sum_{\ell=1}^L W_p^p(\frac{1}{n}\sum_{i=1}^n \delta_{\hat{X}_i^\ell}, \frac{1}{m}\sum_{j=1}^m \delta_{\hat{Y}_j^\ell})$
    \end{algorithmic}
\end{algorithm}

\begin{figure}[t]
    \centering
    \includegraphics[width=0.5\columnwidth]{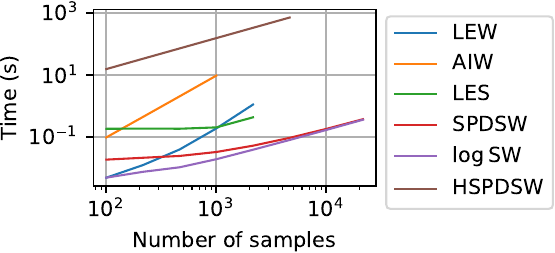}
    \caption{Runtime of $\mathrm{SPDSW}$, $\mathrm{HPDSW}$ and $\mathrm{\log SW}$ (200 proj.) compared to alternatives based on Wasserstein between Wishart samples.
    Sliced discrepancies can scale to larger distributions in $S_d^{++}(\mathbb{R})$.
    }
    \label{fig:runtime_hpdsw}
\end{figure}

\begin{figure}[h!]
    \centering
    \includegraphics[width=\linewidth]{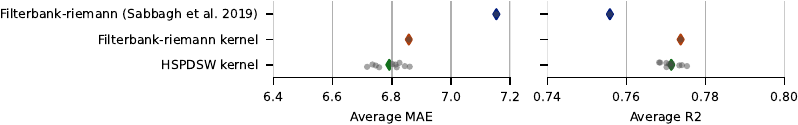}
    \caption{Average MAE and $R^2$ score for 10 seeds on the Cam-CAN dataset with time-frames of 2s and 1000 projections.
    $\mathrm{HSPDSW}$ does not improve the performance of standard methods and the computation time is much higher than for $\mathrm{SPDSW}$.
    }
    \label{fig:brain_age_hspdsw}
\end{figure}

\subsection{Experimental results on brain age prediction.}

As for $\mathrm{SPDSW}$, $\mathrm{HSPDSW}$ allows to define a kernel for distributions $(\mu_i)_{i=1}^{n} \in (\mathcal{P}_p(S_d^{++}(\mathbb{R})))^n$
\begin{equation}
    K(\mu_i, \mu_j) = e^{-\frac{\mathrm{HSPDSW(\mu_i, \mu_j)}}{2\sigma^2}} \enspace .
\end{equation}
Moreover, it is also a Hilbertian pseudo-distance, thus the kernel is positive definite and well-defined for Kernel methods.
Therefore, it can be easily adapted to brain age prediction, as done with $\mathrm{SPDSW}$ in \cref{sec:brain_age}.
The Affine-Invariant metric is well-suited for problems involving source localization, as noted in \citet{sabbagh2019manifold}.
Even though we only have access to the Busemann coordinate, which might involve a loss of information due to the need of an additional projection, it is still of interest to compare to the Log-Euclidean metric.
We report numerical results in the same setting as \cref{fig:brain_age_average} in \cref{fig:brain_age_hspdsw}.
This time, the method does not beat Log-Euclidean Kernel Ridge regression based on the covariance matrices computed over all time samples.
This suggests that the projection $\pi_A$ derived in \cref{sec:busemann_ai} does not bring more information to the model in this scenario.
Note that the computational cost suffers from the high complexity of the $UDU$ decomposition needed for the calculation of each projection.



\end{document}